\PassOptionsToPackage{table}{xcolor}

\makeatletter
\disable@package@load{newtxmath}{}
\makeatother
\documentclass[sigconf,nonacm]{aamas}

\usepackage{amsmath, amssymb, amsthm}
\usepackage{thmtools}
\usepackage{thm-restate}
\usepackage{xspace}
\usepackage{bm}
\usepackage{graphicx}
\usepackage[caption=false,font=footnotesize]{subfig}

\usepackage{array}
\usepackage{makecell}
\newcolumntype{L}[1]{>{\raggedright\arraybackslash}p{#1}}
\newcolumntype{C}[1]{>{\centering\arraybackslash}p{#1}}

\usepackage[table]{xcolor}
\usepackage{graphicx}
\usepackage{booktabs}
\usepackage{float}
\usepackage[ruled,lined,linesnumbered]{algorithm2e}
\DontPrintSemicolon
\RestyleAlgo{ruled}
\SetAlgoLined
\SetKw{Continue}{continue}
\SetKw{Break}{break}

\makeatletter
\AddToHook{cmd/@doclearpage/after}{\global\@colroom\@colht}
\makeatother
\usepackage{paralist}
\usepackage{adjustbox}
\usepackage{wrapfig}
\usepackage{needspace}

\usepackage{xparse}

\usepackage[most]{tcolorbox}
\tcbuselibrary{theorems}

\newtcbtheorem[number within=section]{boxeddefinition}{Definition}{
  enhanced,
  breakable,
  colback=white,
  colframe=black!60,
  boxrule=0.6pt,
  arc=1.2mm,
  left=6pt,right=6pt,top=6pt,bottom=6pt,
  fonttitle=\bfseries,
  coltitle=black,
  attach title to upper,
  boxed title style={
    colback=black!3,
    colframe=black!60,
    boxrule=0.6pt,
    arc=1.2mm
  }
}{def}

\usepackage{stmaryrd} %

\usepackage{tikz}
\usetikzlibrary{shapes.geometric}
\usetikzlibrary{arrows.meta,arrows}
\usetikzlibrary{automata}
\usetikzlibrary{matrix,positioning,backgrounds,fit,calc}

\newcommand{\Ap}{\mathsf{Ap}}
\newcommand{\Alphabet}{\ensuremath{\Gamma}}

\NewDocumentCommand{\MarkovChain}{ O{s} O{\pi} }{%
  \ensuremath{\mathcal{M}_{#2}^{#1}}%
}

\newcommand{\LTL}{\ensuremath{{\textsc{ltl}}\xspace}}
\newcommand{\LTLf}{\ensuremath{{\textsc{ltl}_f}\xspace}}

\newcommand{\PLTLf}{\ensuremath{{\textsc{pltl}_f}\xspace}}

\newcommand{\ATL}{{\ensuremath{\textsc{atl}}\xspace}}

\newcommand{\PATL}{{\ensuremath{\textsc{patl}}\xspace}}

\newcommand{\STIT}{{\ensuremath{\textsc{stit}}\xspace}}

\newcommand{\GDL}{{\ensuremath{\textsc{gdl}}\xspace}}
\newcommand{\ATLsc}{{\ensuremath{\textsc{atl}_\text{sc}}\xspace}}
\newcommand{\ATLes}{{\ensuremath{\textsc{atl}_\text{es}}\xspace}}

\newcommand{\rPATL}{\ensuremath{\text{r}{\textsc{patl}}\xspace}}

\newcommand{\SL}{\ensuremath{{\textsc{sl}}\xspace}}

\newcommand{\Next}{\ensuremath{\mathsf{X}}}
\newcommand{\Until}{\ensuremath{\mathsf{U}}}

\newcommand{\Eventually}{\ensuremath{\mathsf{F}}}
\newcommand{\Always}{\ensuremath{\mathsf{G}}}
\newcommand{\Prob}{\ensuremath{\mathsf{P}}}

\newcommand{\StrategyProfile}{\ensuremath{\vec{\sigma}}}

\newcommand{\TWOEXPTIME}{\ensuremath{\mathsf{2EXPTIME}}}

\newcommand{\defeq}{\stackrel{\mathrm{def}}{=}}

\newtheorem{lemma}{Lemma}

\newtheorem{definition}{Definition}

\newtheorem{assumption}{Assumption}

\renewcommand{\iff}{\ensuremath{\leftrightarrow}}

\newcommand{\FPath}{\ensuremath{\mathsf{FPath}\xspace}}

\newcommand{\NE}{\ensuremath{\mathsf{NE}\xspace}}
\newcommand{\EpsNE}{\ensuremath{\varepsilon\text{-}\mathsf{NE}\xspace}}
\newcommand{\SPNE}{\ensuremath{\mathsf{SPNE}\xspace}}

\newcommand{\NashConv}{\ensuremath{\mathsf{NashConv}\xspace}}
\newcommand{\ApxNashConv}{\ensuremath{\NashConv^{\!\approx}}\xspace}
\newcommand{\Sol}{\ensuremath{\mathsf{Sol}\xspace}}

\newcommand{\Sel}{\ensuremath{\mathsf{Sel}}}
\newcommand{\Exploit}{\ensuremath{\mathsf{Exploit}}}

\renewcommand{\iff}{\ensuremath{\leftrightarrow}}

\newcommand{\Val}{\mathsf{Val}}

\newcommand{\trace}{\tau}

\newcommand{\Obj}{\psi}
\newcommand{\ObjVec}{\vec\Obj}        %
\newcommand{\PathObj}{\varphi}
\newcommand{\PathObjVec}{\vec\PathObj} %

\newcommand{\util}{u}
\newcommand{\utilVec}{\vec\util}

\definecolor{green_hl}{rgb}{0.0, 0.56, 0.0}

\newcommand\env[1]{\textsc{#1}}

\newcommand\algo[1]{\texttt{#1}}

\newlength{\logicsemanticsrhswidth}
\newlength{\logicsemanticslhswidth}
\newlength{\logicsemanticstextinset}
\newcommand{\logicsemanticssetrhswidth}[1]{%
    \setlength{\logicsemanticsrhswidth}{#1}%
}

\newcommand{\logicsemanticssetlhswidth}[1]{%
    \setlength{\logicsemanticslhswidth}{#1}%
}

\newcolumntype{Y}[1]{>{\raggedright\arraybackslash}p{#1}}

\newcommand{\logicsemanticssetboxstyle}[1]{%
    \tcbset{logicsemantics/.style={#1}}%
}

\logicsemanticssetboxstyle{%
    enhanced,
    breakable,
    colframe=black!75!white,
    colback=black!4!white,
    arc=0mm,
    boxrule=0.95pt,
    left=0mm,
    right=0mm,
    top=0.5mm,
    bottom=0.5mm,
    boxsep=0pt
}

\newcommand{\logicsemanticsintrofont}{%
    \footnotesize\sffamily\bfseries
}

\newcommand{\logicsemanticssectionfont}{%
    \footnotesize\sffamily\bfseries
}

\newcommand{\logicsemanticsbodyfont}{%
    \scriptsize
}

\newcommand{\logicsemanticssetbodyfont}[1]{%
    \renewcommand{\logicsemanticsbodyfont}{#1}%
}

\newcommand{\logicsemanticsfullrow}[1]{%
    \multicolumn{3}{@{}Y{\dimexpr\linewidth-2\logicsemanticstextinset\relax}@{}}{%
        \hspace*{\logicsemanticstextinset}\rule{0pt}{2.2ex}#1%
    }\\[-0.15ex]%
}

\providecommand{\logicsemanticsintro}{%
    For any valuation $\nu$, history $h$, path $\pi$, position $i \geq 1$, and formulae of the logic under study:%
}

\newenvironment{logicsemantics}[1][\logicsemanticsintro]{%
    \par\medskip
    \noindent
    \begin{tcolorbox}[logicsemantics]
    \logicsemanticsbodyfont
    \setlength{\tabcolsep}{0pt}%
    \renewcommand{\arraystretch}{0.92}%
    \begin{tabular*}{\linewidth}{@{}Y{\logicsemanticslhswidth} c Y{\logicsemanticsrhswidth}@{}}%
    \if\relax\detokenize{#1}\relax
    \else
        \logicsemanticsfullrow{\logicsemanticsintrofont #1}%
        \noalign{\vskip 0.35ex}%
    \fi
}{%
    \end{tabular*}%
    \end{tcolorbox}%
    \par\medskip
}

\newcommand{\logicsemanticsdivider}{%
    \noalign{\hrule height 0.4pt}%
}

\newcommand{\logicsemanticssection}[1]{%
    \logicsemanticsdivider
    \logicsemanticsfullrow{\logicsemanticssectionfont #1}%
    \logicsemanticsdivider
    \noalign{\vskip 0.4ex}%
}

\newcommand{\logicsemanticsfirstsection}[1]{%
    \logicsemanticsfullrow{\logicsemanticssectionfont #1}%
    \logicsemanticsdivider
    \noalign{\vskip 0.4ex}%
}

\newcommand{\logicsemanticsrule}[3]{%
    #1 & #2 & #3\\%
}

\newcommand{\logicsemanticsruleblock}[3]{%
    #1 & #2 & \begin{tabular}[t]{@{}l@{}}#3\end{tabular}\\%
}

\usepackage[scale=0.9]{sourcecodepro}
\DeclareTextCommandDefault{\copyright}{\textcopyright}

\setcopyright{none}

\title{\texorpdfstring{Learning Rational Strategies\\ for Temporal Objectives in Multi-Agent Systems}{Learning Rational Strategies for Temporal Objectives in Multi-Agent Systems}}

\author{Omar Adalat}
\affiliation{
  \institution{Imperial College London}
  \city{London}
  \country{United Kingdom}}
\email{o.adalat24@imperial.ac.uk}

\author{Edwin Hamel-De le Court}
\affiliation{
  \institution{University of Manchester}
  \city{Manchester}
  \country{United Kingdom}}
\email{edwin.hamel-delecourt@manchester.ac.uk}

\author{Francesco Belardinelli}
\affiliation{
  \institution{Imperial College London}
  \city{London}
  \country{United Kingdom}}
\email{francesco.belardinelli@imperial.ac.uk}

\begin{abstract}
We study rational synthesis for stochastic multi-agent systems with unknown dynamics, where each agent is equipped with its own probabilistic temporal objective specified using Probabilistic Linear Temporal Logic on Finite Traces ($\PLTLf$). We impose solution concepts (e.g., Nash equilibrium, subgame-perfect equilibrium) as constraints on the joint strategy profile, using payoffs derived from each agent's probability of satisfying its temporal objective. In particular, we formulate per-agent rational synthesis as finding a strategy profile -- if one exists -- such that (i) the designated agent's probabilistic temporal constraint is satisfied, and (ii) the profile represents a stable solution under a chosen solution concept. We give conditional convergence guarantees for tabular Nash-Q learning, certify satisfaction with simultaneous finite-sample confidence bounds, and assess strategic stability with approximate best-response checking. Our approach operates in a stochastic \emph{Multi-Agent Reinforcement Learning} (MARL) environment, where model-free learning algorithms attempt to attain extremal satisfaction probabilities for temporal objectives automatically compiled into reward automata. Our approach is empirically validated using 8 general-sum concurrent stochastic games.
\end{abstract}

\keywords{Multi-agent Reinforcement Learning, General-Sum Games, Rational Learning, Probabilistic Temporal Objectives}

\newcommand{\BibTeX}{\rm B\kern-.05em{\sc i\kern-.025em b}\kern-.08em\TeX}

\begin{document}

\maketitle
\pagestyle{empty}
\thispagestyle{empty}

\section{Introduction} \label{sec:intro}

\emph{Strategic reasoning} in multi-agent systems (MAS) hinges on determining what outcomes groups of agents (termed coalitions) can bring about through their interactions. In many multi-agent settings, especially safety-critical systems, properties of interest go beyond plain reward signals and include \emph{temporally extended goals} such as \emph{reachability/liveness} (e.g., eventually reaching a target region or completing a mission), or \emph{safety invariants} (e.g., always avoid unsafe regions, or never violate constraints)~\cite{amodei2016concrete,garcia2015comprehensive}. Similar objectives are naturally captured by temporal logics~\cite{manna1990hierarchy} which provide precise languages and semantics to specify and verify such properties. Languages relevant to specification of strategic properties include Alternating-Time Temporal Logic ($\ATL$~\cite{alur2002alternating}), Dynamic Coalition Logic~\cite{galimullin2021dynamic}, Probabilistic Alternating-$\mu$ Calculus~\cite{song2019probabilistic}, Game Description Logic ($\GDL$~\cite{de2010situation}), and See-To-It-That ($\STIT$~\cite{belnap1988seeing}). However, such logics typically resolve non-determinism by restricting to coalitional modalities which aim to universally and unilaterally (inevitably) force a property to be true.

On the other hand, \emph{rational synthesis}~\cite{fisman2010rational,ummels2006rational} aims to find strategies that are self-enforcing through stable equilibria, instead of quantifying against all possible behaviour (including irrational behaviour) by other agents who have their own objectives. Therefore, rational synthesis is typically less conservative than adversarial synthesis.

Settings we consider can be represented formally as \emph{(general-sum) stochastic games}, which model multiple decision-making agents that interact over time in a dynamic environment with probabilistic behaviours. Each agent repeatedly selects an action, and the combination of agents' non-deterministic choices and the environment's probabilistic transitions determines the evolution of the game. Stochastic games have been applied in domains as diverse as \emph{cybersecurity}~\cite{zhao2020finite}, \emph{epidemics}~\cite{huang2022game}, \emph{auctions}~\cite{manvi2023nash}, and routing in \emph{stochastic traffic networks}~\cite{yang2024markov}. 

We use an expressive specification style in which \emph{each agent} has its own probabilistic temporal objective as a logical formula interpreted on finite traces. Finite traces have a natural correspondence with reinforcement learning (RL) as tasks are often trained and executed episodically. Rationality is captured by imposing that the learned joint strategy profile satisfies a specific solution concept (e.g., ($\varepsilon$-)Nash equilibrium, subgame-perfect equilibrium), where each agent's payoff reflects the preferred direction of its satisfaction probability. Our approach emphasises learning a single committed strategy profile, allowing tractability by not requiring an explicit transition model of the underlying stochastic game which generally grows combinatorially with respect to the number of agents involved.

\paragraph{Contributions.}
Our contributions are as follows:  \begin{inparaenum}[(i)]
  \item We formulate a learning-based rational-synthesis objective for finite-horizon stochastic games in which each agent has a probabilistic $\LTLf$ objective and utilities encode the preferred direction of satisfaction probabilities;
  \item We compile per-agent $\LTLf$ objectives into deterministic reward automata and show that terminal automaton rewards have expectation equal to satisfaction probability;
  \item We provide theoretical analysis for the reward-automaton encoding, discounting effects, finite-sample satisfaction certification, and finite-horizon equilibrium existence and convergence;
  \item We evaluate satisfaction and strategic stability on 8 benchmark environment instances.
\end{inparaenum}

For reasons of space, proofs are deferred to Appendix~\ref{appsec:technical_details}. A high-level overview of our approach is presented in Appendix~\ref{appsec:overview}.

\section{Preliminaries}
\label{sec:prelims}

\paragraph*{Multi-agent Reinforcement Learning.}
We model multi-agent RL problems as an $n$-agent \emph{concurrent stochastic game} (also known as a {\em Markov game}~\cite{littman1994markov}) $\mathcal{M} = (\mathcal{S}, s_0, \mathcal{A}^1, \dots, \mathcal{A}^n,\allowbreak d^1,\dots,d^n,\allowbreak \mathcal{P}, \mathcal{R}^1, \dots, \mathcal{R}^n, \gamma)$, where $\mathcal{S}$ is the finite {\em state space} and $s_0 \in \mathcal{S}$ is the {\em initial state}; $\mathcal{A}^i$ is the finite {\em action space} of agent $i \in [1, \dots, n]$; 
$d^i:\mathcal{S}\to 2^{\mathcal{A}^i}\setminus\{\emptyset\}$ is the \emph{protocol} (action-availability) function of agent $i$, where $d^i(s)$ denotes the set of legal actions for agent $i$ in state $s$; the set of legal joint actions at $s$ is $d(s)=\prod_{i=1}^n d^i(s)\subseteq \mathcal{A}$;
$\mathcal{P}: \mathcal{S} \times \mathcal{A} \to \Delta(\mathcal{S})$ is the (partial) {\em transition function} (defined only for joint actions $a \in d(s)$), with $\Delta(\mathcal{S})$ being the set of probability distributions over $\mathcal{S}$;  $\mathcal{R}^i: \mathcal{S} \times \mathcal{A} \to \mathbb{R}$ is the (Markovian) reward function of agent $i$; and $\gamma$ is the discount factor.\\
A (stochastic) \emph{policy} $\pi^i$ (also known as a {\em strategy}) of agent $i$ is a function $\pi^i : \mathcal{S} \to \Delta(\mathcal{A}^i)$ with $\operatorname{supp}\pi^i(\cdot\mid s)\subseteq d^i(s)$ for every $s$. Then, the joint action space of agents is $\mathcal{A}=\prod_{i=1}^n\mathcal{A}^i$, and the joint policy of other agents is $\pi^{-i}(a^{-i}|s) = \prod_{j \neq i}\pi^j(a^j|s)$, where $a^{-i}$ is the joint action except agent $i$'s action. Each agent $i$ has their own goal to maximise their expected cumulative discounted rewards, defined as:
\begin{equation}
\mathop{\mathbb{E}}\limits_{\substack{s_{t+1}\sim\mathcal{P}(\cdot|s_t, a_t^i, a_t^{-i}), \\a_t^i\sim\pi^i(\cdot |s_t), a_t^{-i}\sim\pi^{-i}(\cdot |s_t)}}\left [ \sum_{t=0}^{\infty}\gamma^t\mathcal{R}^i(s_t,a_t^i,a_t^{-i})  \right ]
\end{equation}
The decentralised stochastic game, from the perspective of each agent, is \emph{non-stationary},  as the other agents concurrently update their policies. General-sum stochastic games have computational solution concepts which assess stability under equilibria, since they account for the actions of the remaining agents.

A \emph{labelled stochastic game} simply extends a stochastic game with a labelling function $\mathcal{L}: \mathcal{S} \to 2^{\Ap}$ which maps each state $s$ to a set $\mathcal{L}(s) \subseteq \Ap$ of atoms which are true at state $s$. We denote a joint set of strategies for a coalition $C \subseteq [1, \dots, n]$ as $\sigma_C$, and for all agents in the concurrent stochastic game, we denote the \emph{strategy profile} $\StrategyProfile \defeq  (\sigma_1, \dots, \sigma_n)$ as the set composed of each agent's strategy.

We denote a \emph{finite path} as $\pi=s_0a_0s_1\cdots a_{T-1}s_T$, where $a_t\in d(s_t)$ and $\mathcal P(s_{t+1}\mid s_t,a_t)>0$ for every $0\le t<T$. The horizon $H$ bounds the number of transitions, so $T\le H$ and the associated label trace has $T+1$ positions. Let $\FPath_{\mathcal M, s}$ denote the set of finite paths from $s$. For a strategy profile $\StrategyProfile$, let $\Pr^{\StrategyProfile}_{\mathcal M,s}$ be the induced probability measure on complete finite paths $\FPath^{\text{comp}}_{\mathcal M,s}$.  We refer to \cite{baier2008principles} for the standard construction of the probability measure $\Pr^{\StrategyProfile}_{\mathcal M,s}$ induced by a strategy profile. We learn a single strategy profile using \algo{IPPO}~\cite{schulman2017proximal,de2020independent}, \algo{IQL}~\cite{Tan1993}, \algo{Nash-Q Learning (NQL)}~\cite{hu2003nash}, and \algo{PR2-IQL}/\algo{PR2-IPPO} variants~\cite{wen2019pr2} (further details in Appendix~\ref{appsec:PR2}).

\paragraph*{Probabilistic Linear Temporal Logic on Finite Traces.}
We use Probabilistic $\LTLf$ ($\PLTLf$)~\cite{maggi2020temporal} to specify each agent's probabilistic temporal objective on finite traces, matching the episodic structure of RL tasks. Each completed episode supplies a finite sequence of state labels on which a temporal goal is evaluated: for example, reaching a target before the episode ends while avoiding collisions throughout. The probabilistic requirement then constrains how often that goal is satisfied under a fixed strategy profile. Equilibrium constraints capture rationality through utilities reflecting each agent's preferred direction of satisfaction.
The grammar of Probabilistic Linear Temporal Logic on Finite Traces ($\PLTLf$) is  defined as follows:
\begin{definition}[$\PLTLf$ Formula]
  $\PLTLf$ formulas are constructed inductively from the set of atoms $\Ap$, where $p \in \Ap$, $\top$ denotes true and $\bot$ denotes false, $q \in [0,1]$ is a rational constant, and $\bowtie \: \in \{ <, \leq, >, \geq \}$. Formally, the grammar is:
    \begin{align*}
      \varphi &::= \top
        \mid p
        \mid \neg \varphi
        \mid \varphi \land \varphi
        \mid \Next \, \varphi
        \mid \varphi_1 \, \Until \, \varphi_2
      \\
      \psi &::= \Prob_{\bowtie q}[\varphi]
    \end{align*}
\end{definition}
$\Next$ denotes temporal Next and $\Until$ denotes temporal Until. The remaining Boolean and temporal operators 
are syntactic sugar defined as usual: Eventually $\Eventually$ as $\Eventually \varphi \equiv \top \Until \varphi$ and Always $\Always$ as $\Always \varphi \equiv \lnot \Eventually \lnot \varphi$.
In multi-agent systems, agents with no specific objectives can be included in a residual coalition with a trivial objective always achievable with probability of 1 (e.g., $\top$). The full finite-trace semantics are given in Appendix~\ref{appsec:pltlf_semantics}. We provide the probabilistic clause of $\PLTLf$ as:
\[
\begin{gathered}
\mathcal M,s,\StrategyProfile \models \Prob_{\bowtie q}[\varphi]
\quad\Longleftrightarrow\\
\Pr^{\StrategyProfile}_{\mathcal M,s}\big(\{\pi\in\FPath^{\mathrm{comp}}_{\mathcal M,s}\mid
\mathcal M,\pi\models \varphi\}\big)\ \bowtie q.
\end{gathered}
\]

\paragraph*{Rational Synthesis.}
Classical logics with strategic coalition modalities (such as $\ATL$~\cite{alur2002alternating} and $\PATL$~\cite{chen2007probabilistic}) assert temporal goals against \emph{all} behaviours of the other agents. This treats other coalitions as strictly adversarial, which can lead to overly conservative solutions. In settings where other agents have their own objectives, it is more natural to assume they act \emph{rationally} and restrict their behaviour to strategy profiles satisfying a solution concept (e.g., (strong) Nash equilibrium, subgame-perfect equilibrium)~\cite{fisman2010rational,condurache2016complexity}. We write $\StrategyProfile \in \Sol_\zeta(s;\utilVec)$ to denote that $\StrategyProfile$ satisfies solution concept $\zeta$ at state $s$ for the payoff vector $\utilVec$, where the utilities encode each agent's preferred direction of satisfaction.

\begin{definition}[Satisfaction Probabilities and Utilities]\label{def:objective_utilities}
For agent $i$ with requirement $\Obj_i=\Prob_{\bowtie_i q_i}[\PathObj_i]$, define $p_i(s,\StrategyProfile)\defeq\Pr_{\mathcal M,s}^{\StrategyProfile}(\PathObj_i)$ and $f_i:[0,1]\to[0,1]$ by
\[
f_i(p)\defeq\begin{cases}
p, & \bowtie_i\in\{>,\ge\},\\
1-p, & \bowtie_i\in\{<,\le\}.
\end{cases}
\]
The utility is $u_i(s,\StrategyProfile)\defeq f_i(p_i(s,\StrategyProfile))$. Write $\utilVec(s,\StrategyProfile)\defeq\big(u_i(s,\StrategyProfile)\big)_{i=1}^{n}$.
\end{definition}
Definition~\ref{def:objective_utilities} makes utility maximisation favour larger satisfaction probabilities for lower-bound requirements and smaller probabilities for upper-bound requirements. We target the following notion of $\varepsilon$-Nash equilibrium:

\begin{definition}[$\varepsilon$-Nash Equilibrium] \label{def:eps_ne}
    A strategy profile $\StrategyProfile=(\sigma_1,\dots,\sigma_n)$ is an \emph{$\varepsilon$-Nash equilibrium} ($\EpsNE$) at state $s$ iff, for every agent $i$ and every alternative strategy $\sigma_i'$, we have 
    $
      u_i\!\left(s,(\sigma_i',\sigma_{-i})\right)\le u_i(s,\StrategyProfile)+\varepsilon,
    $
    where $\sigma_{-i}$ denotes the strategies of all agents except $i$. For $\varepsilon=0$, $\StrategyProfile$ is a \emph{Nash equilibrium} ($\NE$).
\end{definition}
The largest gain any single agent can obtain by deviating unilaterally measures how far a profile is from equilibrium:
\begin{definition}[Exploitability]\label{def:exploitability}
    The \emph{exploitability} of a strategy profile $\StrategyProfile$ at state $s$ is
    $
      \Exploit(s,\StrategyProfile)\defeq \max_i\sup_{\sigma_i'}\big[u_i\big(s,(\sigma_i',\sigma_{-i})\big)-u_i(s,\StrategyProfile)\big].
    $
\end{definition}
By Definitions~\ref{def:eps_ne} and~\ref{def:exploitability}, $\StrategyProfile$ is an $\EpsNE$ at $s$ if and only if $\Exploit(s,\StrategyProfile)\le\varepsilon$.

\paragraph{Problem Statement.}
We consider a finite labelled concurrent stochastic game $\mathcal M$ with initial state $s_0$, finite horizon $H$, and a solution concept $\zeta$, such as $\EpsNE$. The transition probabilities are unknown, and strategies are learned through episodic interaction with the environment. Each agent $i\in\{1,\dots,n\}$ has an $\LTLf$ path objective $\PathObj_i$ and a probabilistic requirement $\Obj_i=\Prob_{\bowtie_i q_i}[\PathObj_i]$. Let $\Sigma_i$ denote agent $i$'s admissible strategy space and $\Sigma\defeq\prod_{i=1}^{n}\Sigma_i$. Each agent seeks an extremal satisfaction probability at equilibrium: a maximum for a lower-bound requirement and a minimum for an upper-bound requirement, while meeting its threshold.

\begin{definition}[Equilibrium Profiles]\label{def:equilibrium_profiles}
The set of admissible profiles satisfying solution concept $\zeta$ for the utilities in Definition~\ref{def:objective_utilities} is
\[
\mathcal E_\zeta\defeq\left\{\StrategyProfile\in\Sigma\;\middle|\;\StrategyProfile\in\Sol_\zeta(s_0;\utilVec)\right\}.
\]
\end{definition}
Definition~\ref{def:equilibrium_profiles} captures stability under the agents' individual payoff-maximising preferences. For agent $i$, the acceptable equilibrium profiles are $\mathcal F_\zeta^i\defeq\{\StrategyProfile\in\mathcal E_\zeta\mid p_i(s_0,\StrategyProfile)\bowtie_i q_i\}$.

\begin{definition}[Rational Strategy Synthesis]\label{def:rational_synthesis}
Given $(\mathcal M,s_0,H,\allowbreak (\Obj_j)_{j=1}^{n},\zeta)$, admissible strategy spaces $(\Sigma_j)_{j=1}^{n}$, and an agent $i$, the \emph{rational strategy synthesis problem for agent $i$} is to find a strategy profile $\StrategyProfile^\star=(\sigma_i^\star,\sigma_{-i}^\star)\in\mathcal F_\zeta^i$, if such a profile exists.
\end{definition}
Definition~\ref{def:rational_synthesis} requires agent $i$'s threshold to hold at an equilibrium profile. The other agents' objectives determine their strategic preferences, but their thresholds do not necessarily need to hold. Correspondingly, the feasibility question is whether $\mathcal F_\zeta^i\neq\emptyset$, since the existence of equilibria alone is insufficient for the existence of threshold-constrained Nash equilibria. Proving infeasibility may require ruling out all admissible equilibrium profiles, which is outside the scope of model-free rational learning.

An equilibrium profile is additionally said to be \emph{jointly satisfying} if it belongs to $\mathcal F_\zeta\defeq\bigcap_{i=1}^{n}\mathcal F_\zeta^i$, so that every agent's probabilistic threshold holds at the same profile. Each agent uses a threshold-constrained Nash equilibrium as its primary solution concept, as opposed to a jointly satisfying Nash equilibrium.

We investigate model-free MARL approaches in which agents concurrently learn strategies towards their respective synthesis objectives. For each learned profile, we assess each agent's probabilistic requirement using finite-sample confidence bounds and the profile's strategic stability using finite-budget approximate best-response search.

\section{Encoding Quantitative Objectives} \label{sec:encoding}
We outline how to translate $\PLTLf$ objectives into Markovian reward signals for applying standard (model-free) multi-agent RL algorithms, while preserving the \emph{single strategy profile witness} requirement of strategy synthesis.

We compile each $\LTLf$ path objective into a deterministic reward automaton, then take the synchronous product of these automata with the stochastic game. Terminal acceptance rewards recover satisfaction probabilities (Theorem~\ref{thm:coherence}), and the product makes the objectives Markovian while preserving payoffs and equilibria (Theorem~\ref{thm:markov_sufficiency_product}). Discounting results quantify how terminal rewards relate to satisfaction probabilities under fixed and early-terminating horizons. Sampling results provide finite-sample certificates for probabilistic thresholds. Algorithm~\ref{alg:approach} summarises the procedure.

\begin{algorithm}[t]
\caption{Learning-based $\PLTLf$ rational synthesis under unknown dynamics}
\label{alg:approach}
\KwIn{Labelled stochastic game \(\mathcal M\); objective profile \(\ObjVec=(\Prob_{\bowtie_i q_i}[\PathObj_i])_{i=1}^n\); tolerance; additional budget $\mathcal{B}$, solution concept $\zeta$}
\KwOut{learned strategy profile \(\widehat{\StrategyProfile}\) and per-agent satisfaction certificates}
\For{each agent \(i\in\{1,\dots,n\}\)}{
    Compile \(\PathObj_i\) into a deterministic reward automaton \(\mathcal A_{\PathObj_i}\)\;
}
Construct the composite reward automaton \(\mathcal A_{\PathObjVec}\)\;
Train with terminal rewards \(\mathcal{R}_T^i=f_i(\mathbf 1[\bm m_T\models F_i])\) and optionally with counterfactual experience trajectories\;
Using fixed profile \(\widehat{\StrategyProfile}\), estimate unilateral-improvement score \(\widehat{\delta}_{\zeta}\) by rollouts\;
\If{$\widehat{\delta}_{\zeta}>\varepsilon_{\zeta}$}{
    Train once more with budget $\mathcal{B}$, then re-estimate \(\widehat{\delta}_{\zeta}\)\;
}
Estimate simultaneous intervals \([L_i,U_i]\) for all objectives from fresh rollouts of the final fixed profile\;
\Return{\(\widehat{\StrategyProfile}\), and for each agent \(i\), whether its threshold condition is certified}
\end{algorithm}

\paragraph{Reward Automata.}
For each agent, a deterministic finite automaton (DFA) tracks whether the labelled trace of an episode satisfies the $\LTLf$ path formula inside its probabilistic constraint. Its state summarises the history relevant to that formula and is updated as each new state label is read. For example, the automaton in Figure~\ref{fig:dfa} remembers whether a breach has occurred and whether a collision has made the objective impossible. At episode end, its accepting states identify successful traces, and averaging these binary outcomes estimates satisfaction probability. Constructing this automaton takes $\TWOEXPTIME$ in the worst case~\cite{de2015synthesis}. In practice, temporal objectives are often simple and avoid deep temporal nesting, allowing existing tools to construct their automata in milliseconds. Thus, the worst-case complexity need not make compilation a practical bottleneck. Figure~\ref{fig:dfa} illustrates a simple example.

\begin{figure}[t]
  \centering
  \definecolor{lightpurple}{rgb}{0.8, 0.8, 1}
\definecolor{darkpurple}{rgb}{0.6, 0.6, 0.9}

\definecolor{lightgreen}{rgb}{0.76, 0.96, 0.76}
\definecolor{darkgreen}{rgb}{0.46, 0.76, 0.46}

\definecolor{lightblue}{rgb}{0.80, 0.88, 1.00}
\definecolor{darkblue}{rgb}{0.25, 0.45, 0.85}

\tikzset{
    mynode/.style={circle, draw=darkblue, fill=lightblue, thick, minimum size=20pt, inner sep=1.5pt},
    myedge/.style={->, >=stealth', thick},
    mylabel/.style={midway, font=\scriptsize, align=center, inner sep=1pt},
    initial glow/.style={
        mynode,
        preaction={
          draw=darkblue,
          line width=3pt,
          opacity=.50
        }
   },
}

\begin{tikzpicture}[auto, font=\small]

  \node[mynode, initial glow] (0) at (0,1.5) {0};
  \node[mynode] (1) at (-1.5,0) {1};
  \node[mynode, double=darkpurple, double distance=1pt] (2) at (1.5,0) {2};

  \draw[myedge] (0) edge[loop above]
    node[mylabel, above, yshift=2pt] {$\lnot \mathsf{collision}$\\$\land\ \lnot \mathsf{intruder\_breach}$} (0);

  \draw[myedge] (0) -- node[mylabel, left] {$\mathsf{collision}$} (1);
  \draw[myedge] (0) -- node[mylabel, right, xshift=3pt, yshift=4pt]
  {$\mathsf{intruder\_breach}\land$\\ $\lnot\mathsf{collision}$} (2);

  \draw[myedge] (2) edge[loop right]
    node[mylabel, right] {$\lnot \mathsf{collision}$} (2);
  \draw[myedge] (2) -- node[mylabel, below] {$\mathsf{collision}$} (1);

  \draw[myedge] (1) edge[loop left]
    node[mylabel, left] {$\top$} (1);

\end{tikzpicture}
  \caption{DFA for $\Eventually(\mathsf{intruder\_breach})\land\Always(\lnot\mathsf{collision})$ in \env{Stochastic Pursuit Evasion}. State $0$ is initial and state $2$ is accepting. Acceptance is evaluated at episode end.}
  \Description{Three-state reward automaton. State 0 waits for an intruder breach without collision, state 2 records success, and any collision leads to the absorbing failure state 1.}
  \label{fig:dfa}
\end{figure}

\begin{definition}[Deterministic Reward Automaton]
  Given a $\LTLf$ formula, the corresponding \emph{reward automaton} is a deterministic finite automaton over the alphabet $\Alphabet\defeq 2^{\Ap}$, represented as the tuple $A_{\varphi} = (M, m_0, \Alphabet, \delta, F)$, where $M$ is the set of states; $m_0 \in M$ is the initial state; $F \subseteq M$ is the set of accepting states; and $\delta: M \times \Alphabet \to M$ is the total transition function.
\end{definition}
For every finite word $\tau\in\Alphabet^+$, we have:
$
\tau,0\models \varphi \; \text{iff} \; \delta^\ast(m_0,\tau)\in F,
$
where $\delta^\ast$ is the usual extension of $\delta$ to words.  The automaton records acceptance at the terminal timestep. Agent $i$'s learning reward applies $f_i$ to this indicator. For each objective $\varphi$ we define an episodic (terminal) reward functional on complete paths:
\[
\mathcal R_\varphi(\pi)\ \defeq \ \mathbf 1\big[\delta^\ast(m_0,\trace(\pi))\in F\big]
\ =\ \mathbf 1\big[\mathcal M,\pi\models \varphi\big].
\]
This is a Bernoulli random variable under the induced measure $\Pr^{\StrategyProfile}_{\mathcal M,s}$.  Theorem~\ref{thm:coherence} states coherence between the expected return induced by a reward automaton and the satisfaction probability.

\begin{restatable}[Coherence]{theorem}{coherencethm}
\label{thm:coherence}
For any strategy profile $\StrategyProfile$ and state $s$:
$
\mathbb E_{\pi\sim\Pr^{\StrategyProfile}_{\mathcal M,s}}\!\left[\mathcal R_\varphi(\pi)\right]
\;=\;
\Pr^{\StrategyProfile}_{\mathcal M,s}(\varphi).
$
\end{restatable}
Proof in Appendix~\ref{appsec:proof_coherence}. For multiple agents, the individual reward automata are run synchronously in a \emph{composite reward automaton}, whose state records each agent's automaton state and whose terminal reward vector records component-wise acceptance.

\begin{definition}[Composite Reward Automaton]
\label{def:composite_reward_automaton}
Let $\PathObj_1,\dots,\PathObj_n\in\LTLf(\Ap)$, and for each $i\in\{1,\dots,n\}$ let
$\mathcal A_{\PathObj_i}=(M_i,m_0^i,\Alphabet,\delta_i,F_i)$ be a deterministic reward automaton.
The \emph{composite reward automaton} is the deterministic monitor
$
\mathcal A_{\PathObjVec}\ \defeq \ \big(M_c,m_0^c,\Alphabet,\delta_c\big)
$
where $M_c\defeq \prod_{i=1}^n M_i$,
$m_0^c\defeq (m_0^1,\dots,m_0^n)$, and
$
\delta_c\big((m_1,\dots,m_n),\ell\big)
\defeq
\big(\delta_1(m_1,\ell),\dots,\delta_n(m_n,\ell)\big).
$
Acceptance is tracked component-wise via the sets $F_i$:
for a composite state $\bm m=(m_1,\dots,m_n)$ we write $\bm m\models F_i$ iff $m_i\in F_i$.
\end{definition}

Definition~\ref{def:composite_reward_automaton} synchronises the individual monitors. Along an episode $\pi=s_0\bm a_0s_1\cdots s_T$, initialise $\bm m_0=\delta_c(m_0^c,\mathcal L(s_0))$ and update $\bm m_{t+1}=\delta_c(\bm m_t,\mathcal L(s_{t+1}))$, so $\bm m_t$ records the label prefix through $s_t$.

\begin{definition}[Time-Augmented Product Game]\label{def:product_stochastic_game}
Let $\mathcal M$ be a finite labelled stochastic game, $\mathcal A_{\PathObjVec}$ its composite reward automaton, and $H$ the horizon. Let $\kappa:\mathcal S\times\{0,\dots,H\}\to\{0,1\}$ encode the original stopping rule, with $\kappa(s,H)=1$ for every $s$. The stopping time is $T\defeq\min\{0\le t\le H:\kappa(s_t,t)=1\}$ (any information needed by this rule is included in $s$). The \emph{time-augmented synchronous product} $\widehat{\mathcal M}_H$ has state space
\[
\widehat{\mathcal S}_H\defeq\mathcal S\times M_c\times\{0,\dots,H\},
\]
where $x=(s,\bm m,t)$ is formed of the environment state, composite monitor state, and elapsed time, with initial state $x_0=(s_0,\delta_c(m_0^c,\mathcal L(s_0)),0)$. A state $x$ is terminal exactly when $\kappa(s,t)=1$. No further action or transition occurs there. At each non-terminal $x$, agent $i$ retains its original legal actions $d_H^i(x)\defeq d^i(s)$, and we write $d_H(x)\defeq\prod_i d_H^i(x)$. For $x'=(s',\bm m',t')$ and $\bm a\in d_H(x)$, the transition kernel is
\[
\tilde{\mathcal P}(x'\mid x,\bm a)\defeq
\begin{cases}
\mathcal P(s'\mid s,\bm a), & \begin{aligned}[t] &\textnormal{if }t'=t+1 \: \text{and}\\ &\phantom{\textnormal{if }}\bm m'=\delta_c(\bm m,\mathcal L(s'));\end{aligned}\\[4pt]
0, & \textnormal{otherwise.}
\end{cases}
\]
The game has a terminal payoff of
$
R(s,\bm m,t)\defeq\big(f_i(\mathbf 1[\bm m\models F_i])\big)_{i=1}^n.
$
\end{definition}
Definition~\ref{def:product_stochastic_game} makes both the temporal objectives and the horizon Markovian. The episode ends at $T$ with the expected terminal payoff being $u_i$ by Theorem~\ref{thm:coherence}. Retaining $t$ allows strategies to depend on the remaining horizon. We take each $\Sigma_i$ to consist of randomised strategies that are memoryless on $(s_t,\bm m_t,t)$, with independent randomisation across agents.
\begin{restatable}[Markov Sufficiency in the Composite Product]{theorem}{markovsufficiencythm}
\label{thm:markov_sufficiency_product}
For finite-horizon $\LTLf$ objectives $\PathObj_1,\dots,\PathObj_n$ and every strategy profile $\StrategyProfile\in\Sigma$,
the satisfaction probabilities $\big(p_i(s_0,\StrategyProfile)\big)_{i=1}^n$ and utilities $\utilVec(s_0,\StrategyProfile)$
can be realised by a memoryless strategy profile on $\widehat{\mathcal M}_H$, and every such memoryless profile induces a profile in $\Sigma$ with the same payoffs.
\end{restatable}
Theorem~\ref{thm:markov_sufficiency_product} identifies the admissible strategies with memoryless product policies while preserving their objective values. Proof is provided in Appendix~\ref{appsec:proof_markov_sufficiency}.

\subsection{Discount Factor and Effective Horizon} \label{sec:discount_factor}
The product above represents undiscounted utilities. Our experiments use terminal rewards with $\gamma\in\{0.99,1\}$ at the original stopping time $T$. The relationship of discounted rewards to satisfaction probability depends on whether episode length is fixed.

The following results apply to terminal acceptance rewards for $\PathObj_i$ under lower-bound requirements and for $\neg\PathObj_i$ under upper-bound requirements. We first consider fixed-length episodes:
\begin{restatable}[Fixed-Horizon Terminal Reward Is Invariant to Discounting]{theorem}{terminalinvariancethm}
\label{thm:terminal_invariance}
Consider episodes of fixed time indices $t = 0, \dots, H$ and a terminal reward of the form $r_t=0$ for $t<H$ and $r_H = \mathbf{1}[\varphi] \in \{0,1\}$.
Then for any strategy profile $\StrategyProfile$,
$
\mathbb{E}_{\StrategyProfile}\!\left[\sum_{t=0}^{H}\gamma^t r_t\right]
= \gamma^{H}\,\Pr_{\StrategyProfile}(\varphi)
$.
In particular, for fixed $H$ and $\gamma\in(0,1)$, any strategy profile that maximises the discounted return also maximises $\Pr_{\StrategyProfile}(\varphi)$. Exact Nash equilibria are unchanged, but exploitability in discounted-return units is $\gamma^H\Exploit$.
\end{restatable}
Proof in Appendix~\ref{appsec:proof_terminal_invariance}. Commonly in RL, when the goal is achieved early or some irrecoverable failure state is reached, the episode terminates and resets. With early termination, maximising discounted return is not exactly maximising satisfaction probability, as it favours earlier satisfaction.

\begin{restatable}[Discounted Return under Early Termination]{proposition}{earlytermdiscountprop}
\label{prop:early_term_discount}
Let $T\le H$ be the (random) termination time of an episode and let $A$ be the acceptance event.
If the reward is terminal, i.e., $r_t=0$ for $t<T$ and $r_T=\mathbf{1}[A]$, then the discounted return is
$
G_\gamma \;=\; \sum_{t=0}^{T}\gamma^t r_t \;=\; \gamma^{T}\mathbf{1}[A].
$
Hence for any strategy profile $\StrategyProfile$, we have
$
\mathbb{E}_{\StrategyProfile}[G_\gamma] \;=\; \mathbb{E}_{\StrategyProfile}\!\big[\gamma^T\mathbf{1}[A]\big]
$.
\end{restatable}
Proof in Appendix~\ref{appsec:proof_early_term_discount}. However, by Corollary~\ref{cor:early_term_bounds}, even with early stopping the discounted terminal-reward objective is a lower-biased surrogate for $\Pr_{\StrategyProfile}(A)$
that additionally encodes a preference for earlier satisfaction:
\begin{restatable}[Discounted Objective Bounds Satisfaction Probability]{corollary}{earlytermboundscor}
\label{cor:early_term_bounds}
Under the assumptions of Prop.~\ref{prop:early_term_discount} and $T\le H$, we have
$
\gamma^H \Pr_{\StrategyProfile}(A)\ \le\ \mathbb{E}_{\StrategyProfile}[G_\gamma]\ \le\ \Pr_{\StrategyProfile}(A).
$
\end{restatable}
Proof in Appendix~\ref{appsec:proof_early_term_bounds}. Thus, when episodes terminate early upon success or failure, the discounted objective ranks policies by a \emph{discount-weighted} satisfaction probability, i.e., it may prefer earlier satisfaction even if the raw probability is unchanged.
If one wishes to optimise \emph{exactly} $\Pr_{\StrategyProfile}(A)$ under early termination, one can set $\gamma=1$, which still yields a finite return in episodic settings. Otherwise, the bias introduced can also be bounded by $\xi$ per Corollary~\ref{cor:gamma_epsilon_discount_bias}:
\begin{restatable}[$\xi$-Bounded Discount Bias under Early Termination]{corollary}{gammaepsilondiscountbiascor}
\label{cor:gamma_epsilon_discount_bias}
Under the assumptions of Corollary~\ref{cor:early_term_bounds}, to bound the discount-induced bias by at most $\xi \in [0,1]$ uniformly over strategy profiles, it suffices to choose
$
\gamma \ge (1-\xi)^{1/H}.
$
That is, for every strategy profile $\StrategyProfile$,
$
0 \le \Pr_{\StrategyProfile}(A)-\mathbb{E}_{\StrategyProfile}[G_\gamma]\le \xi.
$
\end{restatable}
Proof is available in Appendix~\ref{appsec:proof_gamma_epsilon_discount_bias}.
\subsection{PAC-Style Guarantees by Sampling} \label{sec:sampling_guarantees}
We give finite-sample certificates for probabilistic constraints $\Prob_{\bowtie q}[\varphi]$ under a \emph{fixed} strategy profile $\StrategyProfile$ in a finite-horizon setting. Since the dynamics are unknown, we estimate $p$ from rollouts of the fixed profile and bound the estimation error. Such certificates are needed because convergence to an equilibrium does not ensure that probabilistic thresholds are met.

Fix $\mathcal M$, a start state $s$, a path formula $\varphi$, and a strategy profile $\StrategyProfile$. Let $p \ \defeq \ \Pr^{\StrategyProfile}_{\mathcal M,s}(\varphi) \: \in[0,1]$. Consider $n$ i.i.d.\ rollouts (episodes): $\pi_1,\dots,\pi_n \sim \Pr^{\StrategyProfile}_{\mathcal M,s}$ and define the binary satisfaction indicators
$
X_j \ \defeq \ \mathbf 1[\mathcal M,\pi_j\models \varphi] \in\{0,1\},
$
where $j=1,\dots,n$ so that $\mathbb E[X_j]=p$. Let $\hat p \defeq  \frac{1}{n}\sum_{j=1}^n X_j$.

\begin{restatable}[Hoeffding Confidence Interval for Satisfaction Probability]{theorem}{hoeffdingcithm}
\label{thm:hoeffding_ci}
For any confidence parameter $\delta\in(0,1)$, define the half-width
$
\rho_n(\delta) \ \defeq \ \sqrt{\frac{\ln(2/\delta)}{2n}},
$
and the clipped bounds
\[
L_n \ \defeq \ \max\{0,\hat p-\rho_n(\delta)\},
\qquad
U_n \ \defeq \ \min\{1,\hat p+\rho_n(\delta)\}.
\]
Then
$
\Pr\!\big(p \in [L_n,U_n]\big)\ \ge\ 1-\delta,
$
where the probability is over the randomness of the $n$ sampled rollouts.
\end{restatable}
Proof is available in Appendix~\ref{appsec:proof_hoeffding_ci}. Since $\hat p$ is the fraction of the $n$ rollouts that satisfy $\varphi$, Theorem~\ref{thm:hoeffding_ci} states that $|\hat p-p|\le\rho_n(\delta)$ with probability at least $1-\delta$, without knowledge of $\mathcal P$, where $\rho_n(\delta)=O(n^{-1/2})$. 

\begin{restatable}[One-Sided Certification of Threshold Constraints]{corollary}{thresholdcertcor}
\label{cor:threshold_cert}
Fix a threshold $q\in[0,1]$. Under the assumptions of Theorem~\ref{thm:hoeffding_ci}, with probability at least $1-\delta$ over the validation rollouts, all four implications hold:
\[
\begin{aligned}
L_n>q &\ \Longrightarrow\ p>q, & U_n<q &\ \Longrightarrow\ p<q,\\
L_n\ge q &\ \Longrightarrow\ p\ge q, & U_n\le q &\ \Longrightarrow\ p\le q.
\end{aligned}
\]
Thus the probability of issuing any false threshold certificate is at most $\delta$.
\end{restatable}
Proof is available in Appendix~\ref{appsec:proof_threshold_cert}. Corollary~\ref{cor:threshold_cert} certifies $\Prob_{\ge q}[\varphi]$ only when $L_n\ge q$, and symmetrically for upper bounds. If $L_n<q\le U_n$, the sample is inconclusive rather than evidence that $p<q$.

\begin{restatable}[Joint Certification via Bonferroni Adjustment]{corollary}{jointconfidencecertcor}
\label{cor:joint_confidence_cert}
Fix an objective profile
$\ObjVec=(\Obj_1,\dots,\Obj_m)$, where
$\Obj_i=\Prob_{\bowtie_i q_i}[\PathObj_i]$, and a fixed strategy profile $\StrategyProfile$.
Before validation, fix sample sizes $n_i\ge1$ and failure budgets $\delta_i\in(0,1)$ such that $\sum_{i=1}^m\delta_i\le\delta$ for a joint budget $\delta\in(0,1)$.
For each $i$, let
$p_i=\Pr^{\StrategyProfile}_{\mathcal M,s}(\PathObj_i)$ and construct a Hoeffding interval
$[L_i,U_i]$ as in Theorem~\ref{thm:hoeffding_ci}, from $n_i$ i.i.d.\ rollouts and with marginal failure probability $\delta_i$.
Then
\[
\Pr\!\left(\bigcap_{i=1}^m\{p_i\in[L_i,U_i]\}\right)\ge 1-\delta.
\]
Consequently, with probability at least $1-\delta$, all endpoint-based certificates from Corollary~\ref{cor:threshold_cert} are simultaneously sound. In particular, if, for every $i$, the relevant condition
\[
\begin{aligned}
&L_i>q_i \ (\bowtie_i={>}), &\qquad& U_i<q_i \ (\bowtie_i={<}),\\
&L_i\ge q_i \ (\bowtie_i={\ge}), && U_i\le q_i \ (\bowtie_i={\le})
\end{aligned}
\]
holds, then $\mathcal M,s,\StrategyProfile\models\ObjVec$ is jointly certified with confidence at least $1-\delta$.
The uniform Bonferroni choice $\delta_i=\delta/m$ gives half-width
$
\rho_{n_i}(\delta/m)=\sqrt{\ln(2m/\delta)/(2n_i)}.
$
The rollouts used for different objectives need not be independent. In particular, the same rollouts may be evaluated against every $\PathObj_i$.
\end{restatable}
Proof is available in Appendix~\ref{appsec:proof_joint_confidence_cert}. Corollary~\ref{cor:joint_confidence_cert} certifies all agents' constraints simultaneously for a fixed profile, regardless of how it was learned when validation uses fresh i.i.d\ rollouts. 

\paragraph{Boundary Probabilities.}
At $p=q$, strict requirements ($p>q$ or $p<q$) are false, while non-strict requirements ($p\ge q$ or $p\le q$) hold but leave no margin for separating the confidence bound from the threshold using finitely many samples. A useful additional feasibility condition, analogous to~\cite{ohlmann2025synthesis}, is satisfaction with positive slack:
\begin{definition}[Threshold Strengthening]\label{def:threshold_strengthening}
For a positive rational $\eta$, let $T(\psi,\eta)$ replace $\Prob_{\ge q}$ by $\Prob_{\ge q+\eta}$ and $\Prob_{\le q}$ by $\Prob_{\le q-\eta}$ (and similarly for $>$ and $<$), provided the strengthened threshold remains in $[0,1]$.
\end{definition}
Definition~\ref{def:threshold_strengthening} moves each threshold in the more demanding direction.
\begin{assumption}[Generalised Slater Condition]\label{ass:gslater}
There exist a positive rational $\eta$ and a strategy profile $\StrategyProfile^\star\in\Sigma$ such that, for every agent $i$, $T(\Obj_i,\eta)$ is defined and $\mathcal M,s_0,\StrategyProfile^\star\models T(\Obj_i,\eta)$.
\end{assumption}
Assumption~\ref{ass:gslater} expresses joint feasibility with positive slack. It neither places the witness at equilibrium nor ensures that learning finds it. Rather, it ensures the existence of a profile whose thresholds hold with a positive certification margin.

\subsection{Existence of Equilibria} \label{sec:equilibria_existence}
For infinite-horizon concurrent stochastic games with qualitative objectives, equilibrium existence is only known for certain low levels of the Borel hierarchy (e.g., safety objectives). For example, existence of $\EpsNE$ for general $n$-player concurrent reachability games remains an open question~\cite{Chatterjee2007ErrataCSL04}. In the finite-horizon case, however, existence is classical: every finite extensive-form game with perfect recall has a subgame-perfect equilibrium~\cite{selten1975reexamination}, and a Markov one is obtained by backward induction, solving at each state and remaining time the one-shot stage game whose payoffs are the continuation values, which has a mixed equilibrium by Nash's theorem~\cite{Nash1950} (see e.g.,~\cite{fudenberg1991game}). We restate this standard result for our setting in Proposition~\ref{prp:spne_finite_horizon}:
\begin{restatable}[Existence of $\SPNE$ in Finite Horizon~\cite{selten1975reexamination}]{proposition}{spnefinitehorizonprop}
\label{prp:spne_finite_horizon}
Every fully observed concurrent stochastic game with finite state and action spaces, Markovian stage and terminal rewards, and horizon bounded by $H$ admits a \emph{subgame-perfect Nash equilibrium} ($\SPNE$) in randomised strategies that is \emph{Markov}, i.e.\ each agent's strategy depends only on the current state and the remaining time $H-t$. In particular, this holds for the fully observed product game $\widehat{\mathcal M}_H$ with utilities from Definition~\ref{def:objective_utilities}.
\end{restatable}
Proof in Appendix~\ref{appsec:proof_spne_finite_horizon}. Proposition~\ref{prp:spne_finite_horizon} ensures an equilibrium in the unrestricted class of randomised policies over the fully observed product state. However, existence does not guarantee feasibility for agent $i$ in Definition~\ref{def:rational_synthesis}: $\mathcal E_\zeta$ may be nonempty while $\mathcal F_\zeta^i$ is empty.

\subsection{Convergence Guarantees}
Many of the methods considered in our experiments, including \algo{IPPO}, \algo{IQL}, and \algo{PR2}, have no general guarantee of convergence to Nash equilibria under non-linear function approximation and concurrent policy updates~\cite{tsitsiklis1996analysis}. We instead analyse tabular \algo{Nash-Q}~\cite{hu2003nash}, which backs up equilibrium payoffs of successor-state games.

In finite horizon, convergence proceeds backwards from terminal outcomes as continuation values stabilise. The subtlety is equilibrium selection: when several Nash equilibria exist, small estimation errors may change which one \algo{Nash-Q} chooses. We therefore require vanishing errors in the stage-game payoffs to yield vanishing errors in the selected continuation values.

We use the fully observed product $\widehat{\mathcal M}_H$ of Definition~\ref{def:product_stochastic_game}. A state is \emph{reachable} if it has positive probability under some legal joint action sequence from $x_0$. At each non-terminal $x$, a stage game $g$ assigns a payoff vector in $[0,1]^n$ to each $\bm a\in d_H(x)$. In particular, $Q(x)$ denotes the game $\bm a\mapsto Q(x,\bm a)$.
\begin{definition}[Selector and Selector Profile]\label{def:selector_profile}
A \emph{selector} $\Sel$ maps each non-terminal state $x$ and stage game $g$ at $x$ to a Nash equilibrium $\Sel_x(g)$ of $g$, with value $\Val_x(g)\defeq \mathbb E_{\bm a\sim\Sel_x(g)}[g(\bm a)]\in[0,1]^n$. Its \emph{selector profile} $\StrategyProfile_{\Sel}$ is obtained by backward induction: $V^{\Sel}(x)\defeq R(x)$ for terminal $x$, and, at each non-terminal $x$ in time layers $t=H-1,\dots,0$,
\[
\begin{aligned}
Q^{\Sel}(x,\bm a)&\defeq \sum_{x'}\tilde{\mathcal P}(x'\mid x,\bm a)\,V^{\Sel}(x'),\\
\StrategyProfile_{\Sel}(x)&\defeq \Sel_x\big(Q^{\Sel}(x)\big),\qquad
V^{\Sel}(x)\defeq \Val_x\big(Q^{\Sel}(x)\big).
\end{aligned}
\]
\end{definition}
Definition~\ref{def:selector_profile} states the equilibrium choices used throughout backward induction. Intuitively, the selector fixes one Nash-equilibrium continuation at each state when several are available, and backward induction propagates the corresponding payoff vector to earlier states. In our experiments, we use a selector that ranks the equilibria returned by the stage-game solver by their total expected Q-value. Ties favour the equilibrium with the highest minimum individual value.

\algo{Nash-Q} with selector $\Sel$ maintains $Q^k(x,\bm a)\in\mathbb R^n$ and, after observing a transition $(x,\bm a,x')$ at iteration $k$, updates
$
Q^{k+1}(x,\bm a)=\big(1-\alpha_k(x,\bm a)\big)Q^k(x,\bm a)+\alpha_k(x,\bm a)\,V^k(x'),
$
where $V^k(x')=R(x')$ if $x'$ is terminal and $V^k(x')=\Val_{x'}(Q^k(x'))$ otherwise. All other entries are unchanged ($\alpha_k=0$). At non-terminal states, the extracted profile is $\StrategyProfile_k(x)\defeq \Sel_x(Q^k(x))$.

\begin{restatable}[Convergence of Nash-Q Values and Exploitability]{theorem}{nashqconvergencethm}
\label{thm:nashq_convergence}
Let $\Sel$ be a selector, and suppose that (i) $Q^0(x,\bm a)\in[0,1]^n$ for all $(x,\bm a)$; (ii) for every reachable non-terminal $x$ and $\bm a\in d_H(x)$, the step sizes satisfy $\alpha_k(x,\bm a)\in[0,1]$, $\sum_k\alpha_k(x,\bm a)=\infty$ and $\sum_k\alpha_k(x,\bm a)^2<\infty$ almost surely; (iii) each $\Sel_x$ is Borel measurable, the updated pair and step sizes are measurable functions of the past, and, conditional on that past, the successor is sampled from $\tilde{\mathcal P}(\cdot\mid x,\bm a)$; and (iv) for every reachable non-terminal $x$, $\Val_x$ is continuous at $Q^{\Sel}(x)$, which holds in particular if $\Sel_x$ is continuous at $Q^{\Sel}(x)$ or if all Nash equilibria of $Q^{\Sel}(x)$ have the same value vector. Then $\StrategyProfile_{\Sel}$ is a Markov $\SPNE$ of $\widehat{\mathcal M}_H$ and, almost surely, $Q^k\to Q^{\Sel}$ on all reachable non-terminal states, $u_i(s_0,\StrategyProfile_k)\to u_i(s_0,\StrategyProfile_{\Sel})$ for every agent $i$, and $\Exploit(s_0,\StrategyProfile_k)\to0$ (Definition~\ref{def:exploitability}). In particular, for every $\varepsilon>0$, $\StrategyProfile_k$ is an $\EpsNE$ for all sufficiently large $k$. If moreover $\Sel_x$ is continuous at $Q^{\Sel}(x)$ for every reachable non-terminal $x$, then almost surely $\StrategyProfile_k\to\StrategyProfile_{\Sel}$ on all reachable non-terminal states.
\end{restatable}
Proof in Appendix~\ref{appsec:proof_nashq_convergence}. Theorem~\ref{thm:nashq_convergence} assumes that the selected payoff map $\Val_x$ is continuous at $Q^{\Sel}(x)$. Together with (i)--(iii), this yields convergence of Q-values and payoffs and vanishing exploitability. The stronger assumption that $\Sel_x$ is continuous there also yields convergence of action probabilities. Payoff continuity alone allows switching between equilibria with the same payoff vector.

Since ties are broken only after maximising the total expected Q-value, the resulting backup may satisfy fewer thresholds than another available equilibrium if high values for some agents compensate for low values for others. Hence, equilibrium convergence does not ensure satisfaction of the probabilistic thresholds, which must be checked using Corollary~\ref{cor:threshold_cert} (or Corollary~\ref{cor:joint_confidence_cert} for joint certification).

\section{Experimental Analysis}
\label{sec:experimental_eval}

We evaluate Algorithm~\ref{alg:approach}, targeting $\zeta=\EpsNE$ with $\varepsilon=0.05$ (Definition~\ref{def:eps_ne}) as a practical tolerance for unilateral improvements. Improvements below this tolerance are treated as practically negligible relative to finite-sample payoff estimation, approximate best-response training, and stochastic learning dynamics. Appendix~\ref{appsec:hyperparams} provides full hyperparameters, and our code is available at \url{https://github.com/nightly/learning-rational-strategies} for reproducibility.

\newcommand{\MissingLearningCurve}[2][\linewidth]{%
  \fbox{\parbox[c][0.65in][c]{\dimexpr#1-2\fboxsep-2\fboxrule\relax}{\centering\footnotesize Figure unavailable}}%
}

\newif\ifPdfOutput
\ifx\pdfoutput\undefined
  \PdfOutputfalse
\else
  \ifnum\pdfoutput>0
    \PdfOutputtrue
  \else
    \PdfOutputfalse
  \fi
\fi

\newcommand{\PngGraphic}[3][\linewidth]{%
  \IfFileExists{#3}{%
    \ifPdfOutput
      \includegraphics[width=#1]{#3}%
    \else
      \includegraphics[bb=#2,width=#1]{#3}%
    \fi
  }{\MissingLearningCurve[#1]{#3}}%
}

\newcommand{\LearningCurveGraphic}[2][\linewidth]{%
  \PngGraphic[#1]{0 0 343 196}{#2}%
}

\newcommand{\LearningCurveLegend}{%
  \PngGraphic[0.98\linewidth]{0 0 1200 120}{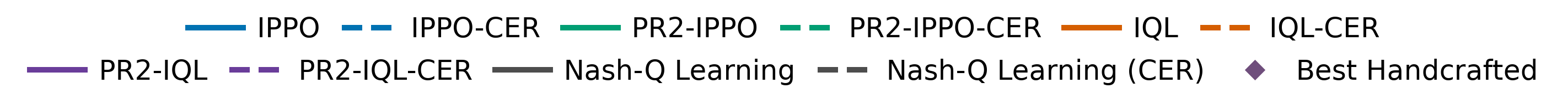}%
}

\newcommand{\AgentCurveBlockWidth}{0.68\textwidth}

\newcommand{\PursuitAgentLearningCurves}[1]{%
  \subfloat[Guard (\texttt{player\_0}), best algorithms]{%
    \LearningCurveGraphic[0.49\linewidth]{#1/agent_1_best_algos_tightened.png}%
  }\hfill
  \subfloat[Guard (\texttt{player\_0}), all algorithms]{%
    \LearningCurveGraphic[0.49\linewidth]{#1/agent_1_all_algos_tightened.png}%
  }\\[0.45em]
  \subfloat[Intruder (\texttt{player\_1}), best algorithms]{%
    \LearningCurveGraphic[0.49\linewidth]{#1/agent_2_best_algos_tightened.png}%
  }\hfill
  \subfloat[Intruder (\texttt{player\_1}), all algorithms]{%
    \LearningCurveGraphic[0.49\linewidth]{#1/agent_2_all_algos_tightened.png}%
  }%
}

\newcommand{\PursuitAgentLearningFigure}[3]{%
  \begin{figure*}[tp]
    \centering
    \begin{minipage}{\AgentCurveBlockWidth}
    \centering
    \LearningCurveLegend

    \medskip

    \PursuitAgentLearningCurves{#2}
    \end{minipage}
    \caption{Two-player Pursuit #1: per-agent satisfaction (mean and $95\%$ CIs over five seeds). Left: best variant per learner family; right: all objective-trained variants.}
    \Description{Learning curves for each Pursuit-Evasion role, comparing the best-performing algorithms on the left and all evaluated algorithms on the right, with confidence bands.}
    \label{fig:app_pursuit_#3_agent_curves}
  \end{figure*}
}

\newcommand{\PursuitThreeAgentLearningCurves}[1]{%
  \subfloat[Guard (\texttt{player\_0}), best algorithms]{%
    \LearningCurveGraphic[0.49\linewidth]{#1/summary/agent_1_best_algos_tightened.png}%
  }\hfill
  \subfloat[Guard (\texttt{player\_0}), all algorithms]{%
    \LearningCurveGraphic[0.49\linewidth]{#1/summary/agent_1_all_algos_tightened.png}%
  }\\[0.45em]
  \subfloat[Intruder 1 (\texttt{player\_1}), best algorithms]{%
    \LearningCurveGraphic[0.49\linewidth]{#1/summary/agent_2_best_algos_tightened.png}%
  }\hfill
  \subfloat[Intruder 1 (\texttt{player\_1}), all algorithms]{%
    \LearningCurveGraphic[0.49\linewidth]{#1/summary/agent_2_all_algos_tightened.png}%
  }\\[0.45em]
  \subfloat[Intruder 2 (\texttt{player\_2}), best algorithms]{%
    \LearningCurveGraphic[0.49\linewidth]{#1/summary/agent_3_best_algos_tightened.png}%
  }\hfill
  \subfloat[Intruder 2 (\texttt{player\_2}), all algorithms]{%
    \LearningCurveGraphic[0.49\linewidth]{#1/summary/agent_3_all_algos_tightened.png}%
  }%
}

\newcommand{\PursuitThreeAgentLearningFigure}[3]{%
  \begin{figure*}[tp]
    \centering
    \begin{minipage}{\AgentCurveBlockWidth}
    \centering
    \LearningCurveLegend

    \medskip

    \PursuitThreeAgentLearningCurves{#2}
    \end{minipage}
    \caption{Three-player Pursuit #1: per-agent satisfaction (mean and $95\%$ CIs over five seeds). Left: best variant per learner family; right: all objective-trained variants.}
    \Description{Learning curves for each Pursuit-Evasion role, comparing the best-performing algorithms on the left and all evaluated algorithms on the right, with confidence bands.}
    \label{fig:app_pursuit_#3_agent_curves}
  \end{figure*}
}

\newcommand{\StagHuntAgentLearningCurves}[1]{%
  \subfloat[\texttt{player\_0}, best algorithms]{%
    \LearningCurveGraphic[0.49\linewidth]{#1/agent_1_best_algos_tightened.png}%
  }\hfill
  \subfloat[\texttt{player\_0}, all algorithms]{%
    \LearningCurveGraphic[0.49\linewidth]{#1/agent_1_all_algos_tightened.png}%
  }\\[0.45em]
  \subfloat[\texttt{player\_1}, best algorithms]{%
    \LearningCurveGraphic[0.49\linewidth]{#1/agent_2_best_algos_tightened.png}%
  }\hfill
  \subfloat[\texttt{player\_1}, all algorithms]{%
    \LearningCurveGraphic[0.49\linewidth]{#1/agent_2_all_algos_tightened.png}%
  }%
}

\newcommand{\StagHuntAgentLearningFigure}[1]{%
  \begin{figure*}[tp]
    \centering
    \begin{minipage}{\AgentCurveBlockWidth}
    \centering
    \LearningCurveLegend

    \medskip

    \StagHuntAgentLearningCurves{#1}
    \end{minipage}
    \caption{Matrix Stag Hunt: per-agent satisfaction (mean and $95\%$ CIs over five seeds). Left: best variant per learner family; right: all objective-trained variants.}
    \Description{Learning curves for both Stag Hunt players, comparing the best-performing algorithms on the left and all evaluated algorithms on the right, with confidence bands.}
    \label{fig:app_stag_hunt_agent_curves}
  \end{figure*}
}

\newcommand{\CoinsAgentLearningCurves}[1]{%
  \subfloat[\texttt{player\_0}, best algorithms]{%
    \LearningCurveGraphic[0.49\linewidth]{#1/summary/agent_1_best_algos_tightened.png}%
  }\hfill
  \subfloat[\texttt{player\_0}, all algorithms]{%
    \LearningCurveGraphic[0.49\linewidth]{#1/summary/agent_1_all_algos_tightened.png}%
  }\\[0.45em]
  \subfloat[\texttt{player\_1}, best algorithms]{%
    \LearningCurveGraphic[0.49\linewidth]{#1/summary/agent_2_best_algos_tightened.png}%
  }\hfill
  \subfloat[\texttt{player\_1}, all algorithms]{%
    \LearningCurveGraphic[0.49\linewidth]{#1/summary/agent_2_all_algos_tightened.png}%
  }%
}

\newcommand{\CoinsAgentLearningFigure}[1]{%
  \begin{figure*}[tp]
    \centering
    \begin{minipage}{\AgentCurveBlockWidth}
    \centering
    \LearningCurveLegend

    \medskip

    \CoinsAgentLearningCurves{#1}
    \end{minipage}
    \caption{Coins: per-agent satisfaction (mean and $95\%$ CIs over five seeds). Left: best variant per learner family; right: all objective-trained variants.}
    \Description{Per-agent learning curves for both Coins players, comparing the best-performing algorithms on the left and all evaluated algorithms on the right, with confidence bands.}
    \label{fig:app_coins_agent_curves}
  \end{figure*}
}

\newcommand{\TerritoryAgentLearningCurves}[1]{%
  \subfloat[Competitor (\texttt{player\_0}), best algorithms]{%
    \LearningCurveGraphic[0.49\linewidth]{#1/agent_1_best_algos_tightened.png}%
  }\hfill
  \subfloat[Competitor (\texttt{player\_0}), all algorithms]{%
    \LearningCurveGraphic[0.49\linewidth]{#1/agent_1_all_algos_tightened.png}%
  }\\[0.45em]
  \subfloat[Competitor (\texttt{player\_1}), best algorithms]{%
    \LearningCurveGraphic[0.49\linewidth]{#1/agent_2_best_algos_tightened.png}%
  }\hfill
  \subfloat[Competitor (\texttt{player\_1}), all algorithms]{%
    \LearningCurveGraphic[0.49\linewidth]{#1/agent_2_all_algos_tightened.png}%
  }%
}

\newcommand{\TerritoryAgentLearningFigure}[1]{%
  \begin{figure*}[tp]
    \centering
    \begin{minipage}{\AgentCurveBlockWidth}
    \centering
    \LearningCurveLegend

    \medskip

    \TerritoryAgentLearningCurves{#1}
    \end{minipage}
    \caption{Territory: per-agent satisfaction (mean and $95\%$ CIs over five seeds). Left: best variant per learner family; right: all objective-trained variants.}
    \Description{Learning curves for both Territory competitors, comparing the best-performing algorithms on the left and all evaluated algorithms on the right, with confidence bands.}
    \label{fig:app_territory_agent_curves}
  \end{figure*}
}

\newcommand{\HarvestAgentLearningCurves}[1]{%
  \subfloat[\texttt{player\_0}, best algorithms]{%
    \LearningCurveGraphic[0.49\linewidth]{#1/agent_1_best_algos_tightened.png}%
  }\hfill
  \subfloat[\texttt{player\_0}, all algorithms]{%
    \LearningCurveGraphic[0.49\linewidth]{#1/agent_1_all_algos_tightened.png}%
  }\\[0.45em]
  \subfloat[\texttt{player\_1}, best algorithms]{%
    \LearningCurveGraphic[0.49\linewidth]{#1/agent_2_best_algos_tightened.png}%
  }\hfill
  \subfloat[\texttt{player\_1}, all algorithms]{%
    \LearningCurveGraphic[0.49\linewidth]{#1/agent_2_all_algos_tightened.png}%
  }\\[0.45em]
  \subfloat[\texttt{player\_2}, best algorithms]{%
    \LearningCurveGraphic[0.49\linewidth]{#1/agent_3_best_algos_tightened.png}%
  }\hfill
  \subfloat[\texttt{player\_2}, all algorithms]{%
    \LearningCurveGraphic[0.49\linewidth]{#1/agent_3_all_algos_tightened.png}%
  }\\[0.45em]
  \subfloat[\texttt{player\_3}, best algorithms]{%
    \LearningCurveGraphic[0.49\linewidth]{#1/agent_4_best_algos_tightened.png}%
  }\hfill
  \subfloat[\texttt{player\_3}, all algorithms]{%
    \LearningCurveGraphic[0.49\linewidth]{#1/agent_4_all_algos_tightened.png}%
  }%
}

\newcommand{\HarvestAgentLearningFigure}[1]{%
  \begin{figure*}[tp]
    \centering
    \begin{minipage}{0.46\textwidth}%
    \centering
    \LearningCurveLegend

    \medskip

    \HarvestAgentLearningCurves{#1}
    \end{minipage}
    \caption{Commons Harvest: per-agent satisfaction (mean and $95\%$ CIs over five seeds). Left: best variant per learner family; right: all objective-trained variants.}
    \Description{Learning curves for the four Commons Harvest agents, comparing the best-performing algorithms on the left and all evaluated algorithms on the right, with confidence bands.}
    \label{fig:app_commons_harvest_agent_curves}
  \end{figure*}
}

\begin{figure*}[t]
  \centering
  \begin{minipage}{\textwidth}
  \centering
  \LearningCurveLegend

  \smallskip

  \subfloat[Pursuit Classic 2]{\LearningCurveGraphic[0.24\linewidth]{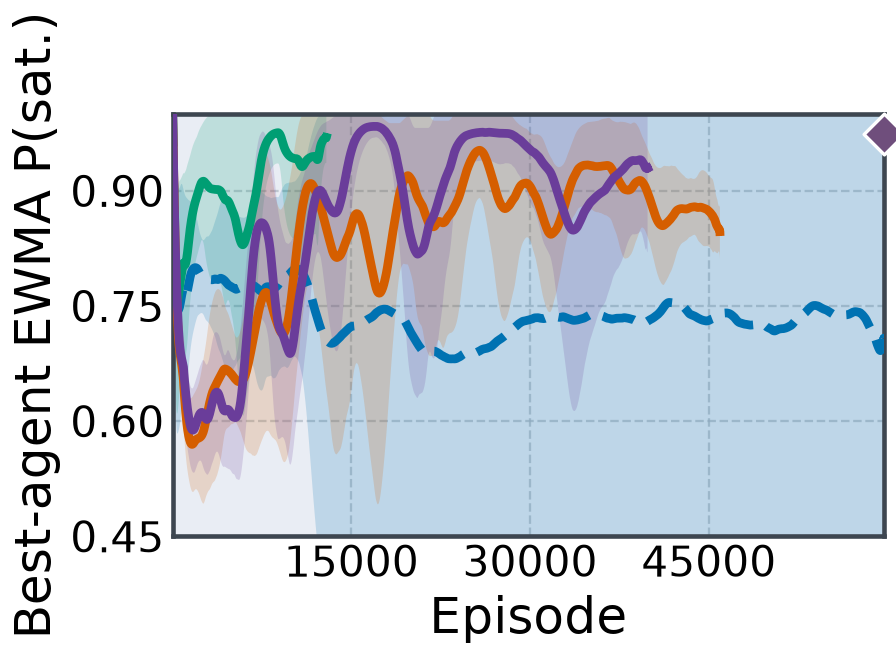}}\hfill
  \subfloat[Pursuit Capture 2]{\LearningCurveGraphic[0.24\linewidth]{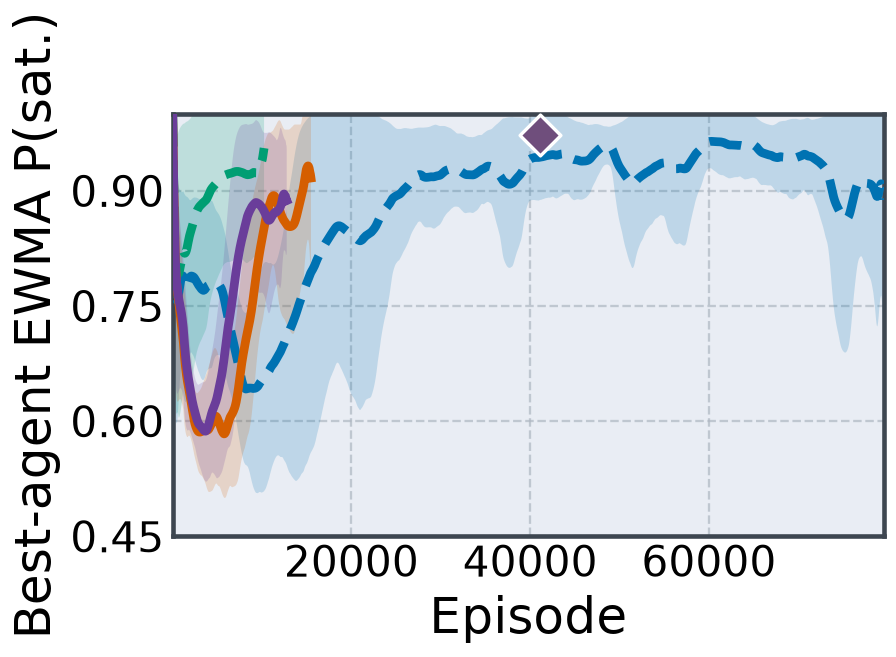}}\hfill
  \subfloat[Pursuit Classic 3]{\LearningCurveGraphic[0.24\linewidth]{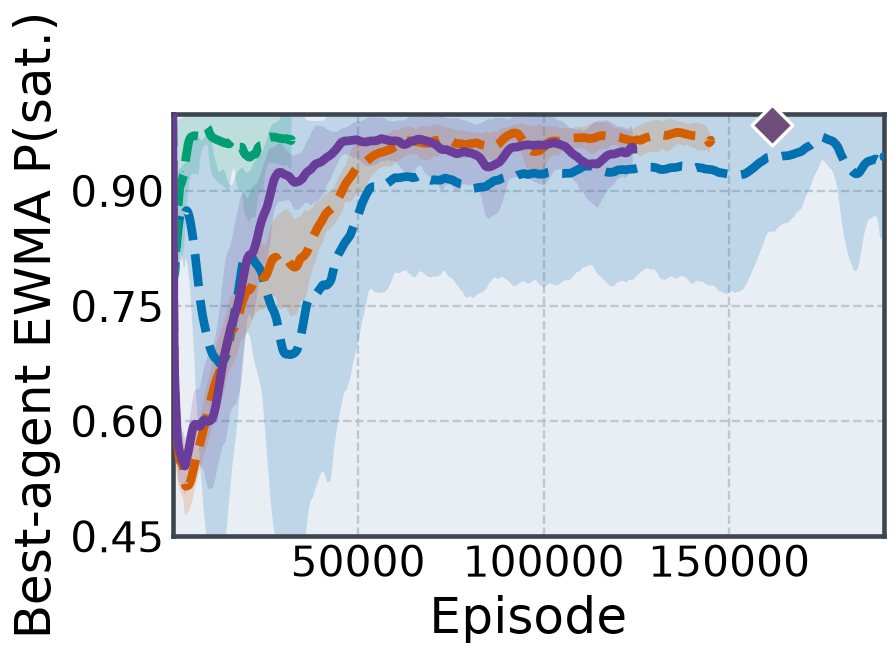}}\hfill
  \subfloat[Pursuit Capture 3]{\LearningCurveGraphic[0.24\linewidth]{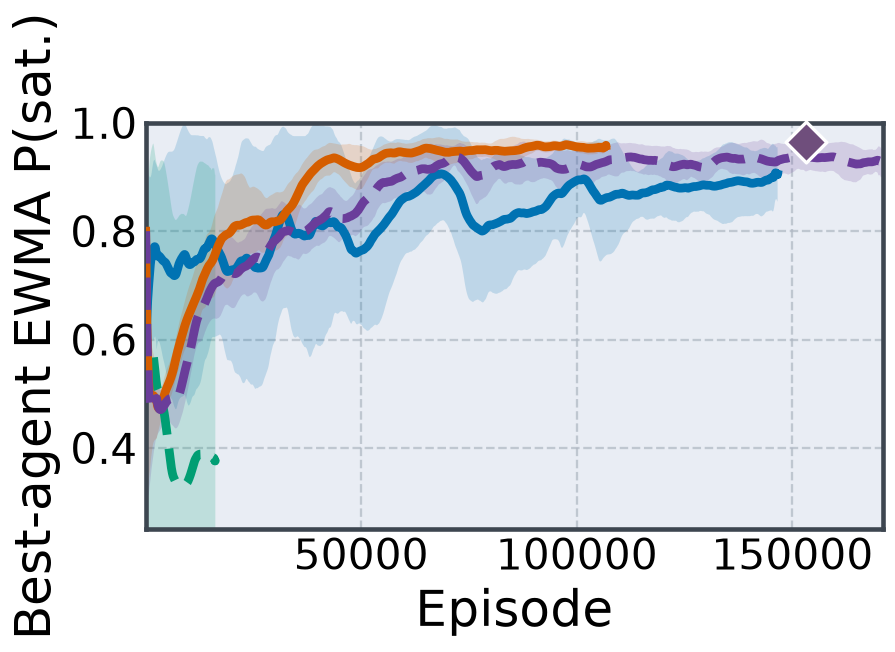}}\\[0.35em]
  \subfloat[\env{Coins}]{\LearningCurveGraphic[0.24\linewidth]{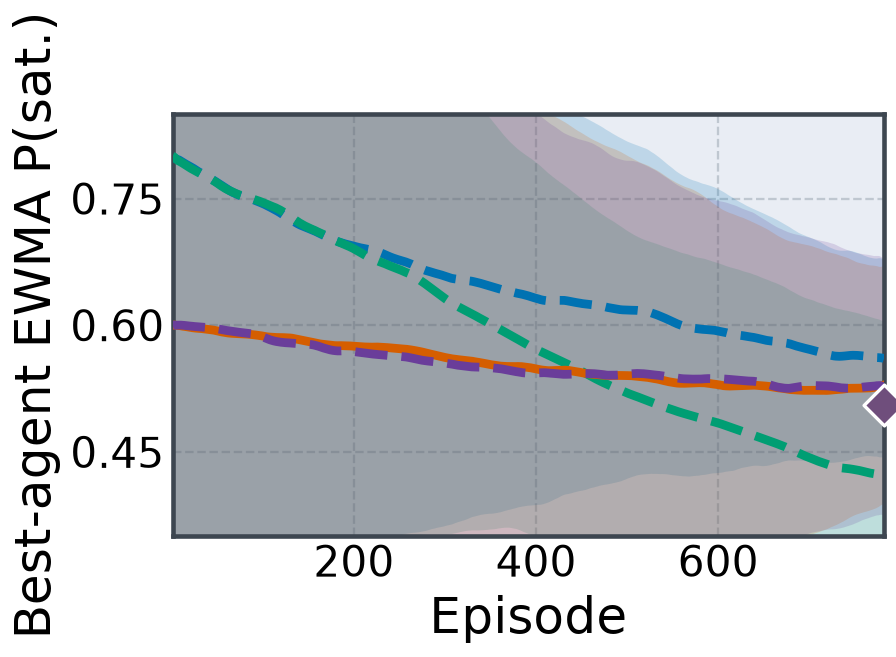}}\hfill
  \subfloat[\env{Matrix Stag Hunt}]{\LearningCurveGraphic[0.24\linewidth]{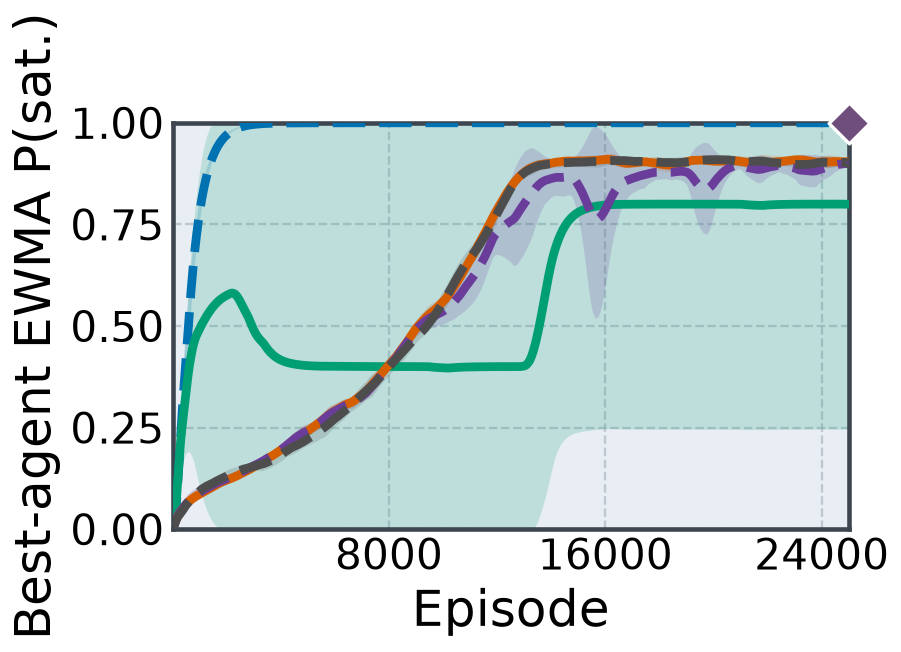}}\hfill
  \subfloat[\env{Commons Harvest}]{\LearningCurveGraphic[0.24\linewidth]{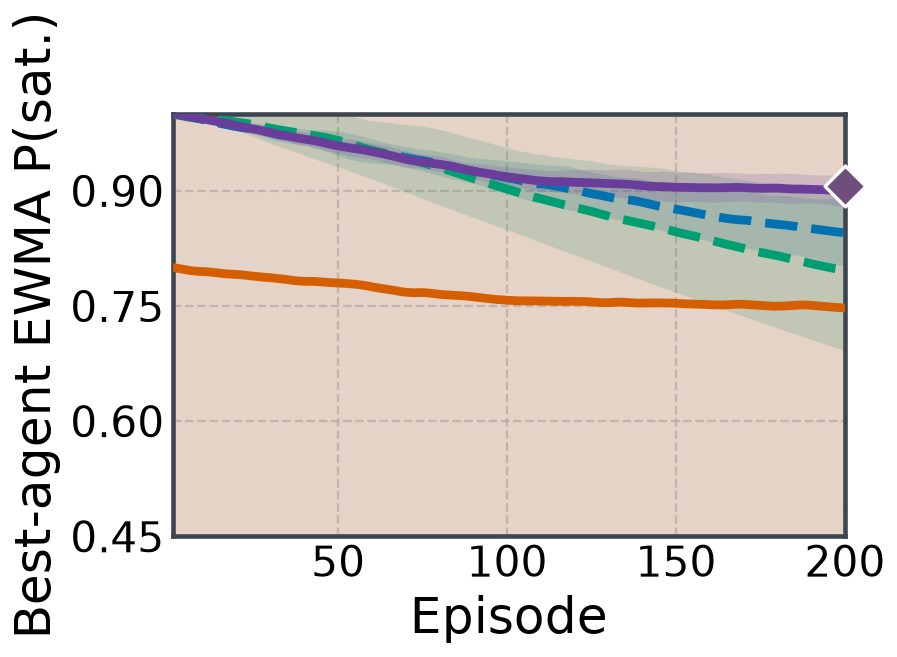}}\hfill
  \subfloat[\env{Territory}]{\LearningCurveGraphic[0.24\linewidth]{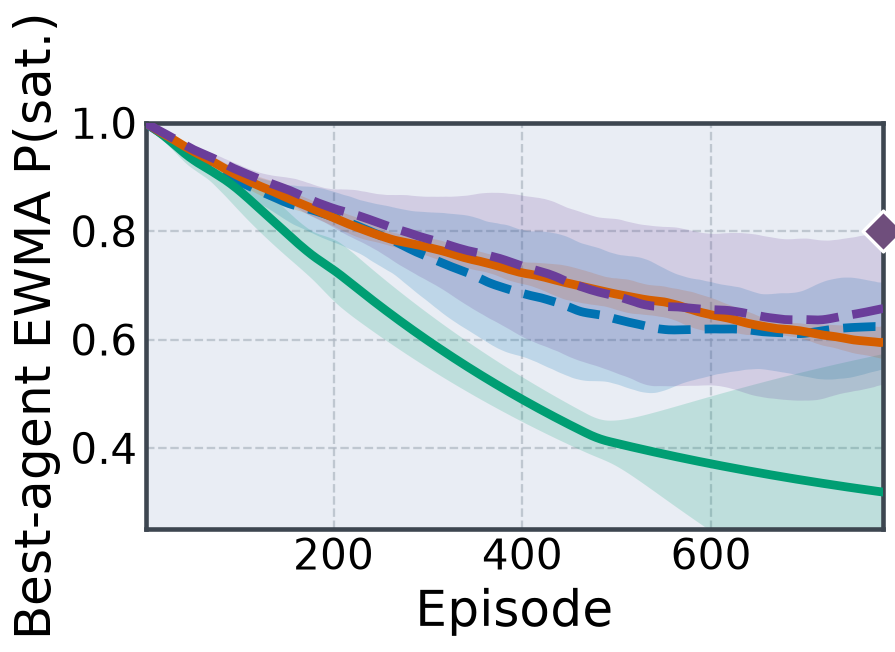}}
  \end{minipage}
  \caption{Best-agent satisfaction of each family's stronger variant, with or without CER (mean and $95\%$ CIs over five runs). Diamonds mark the strongest handcrafted-reward baseline's separately evaluated final satisfaction at its final training step.}
  \Description{Eight panels compare best-agent satisfaction over training across the benchmarks. Bands show confidence intervals, and diamonds mark handcrafted-reward baselines.}
  \label{fig:app_best_agent_learning}
\end{figure*}

\begin{table*}[t]
\centering
\small
\setlength{\tabcolsep}{3pt}
\renewcommand{\arraystretch}{1.15}
\rowcolors{2}{black!3}{white}
\begin{adjustbox}{max width=\textwidth,center}
\begin{tabular}{@{} L{0.15\textwidth} L{0.11\textwidth} C{0.025\textwidth} L{0.45\textwidth} C{0.045\textwidth} C{0.1\textwidth} C{0.081\textwidth} @{}}
\toprule
\textbf{Benchmark} & \textbf{Learner} & \textbf{$i$} & \textbf{Agent's objective} & \textbf{$\widehat p_i$} & \textbf{95\% CI} & \textbf{$\ApxNashConv$} \\
\midrule
Pursuit Classic 2 & \algo{PR2-IPPO} & $0$ & $\Prob_{\ge0.1}[\Always\neg\mathtt{goal\_reached}]$ & $0.979$ & $[0.936,0.993]$ & $0.000$ \\
Pursuit Capture 2 & \algo{PR2-IPPO} & $1$ & $\Prob_{\ge0.5}[\Always\neg\mathtt{guard\_intruder\_collision}\land\Always\neg\mathtt{collision\_ever}]$ & $1.000$ & $[0.957,1.000]$ & $0.000$ \\
Pursuit Classic 3 & \algo{PR2-IQL-CER} & $2$ & $\Prob_{\ge0.1}[\Eventually\mathtt{goal\_reached}\land\Always\neg\mathtt{collision\_ever}]$ & $0.975$ & $[0.932,0.993]$ & $0.008$ \\
Pursuit Capture 3 & \algo{IQL-CER} & $2$ & $\Prob_{\ge0.8}[\Always\neg\mathtt{guard\_intruder\_collision}\land\Always\neg\mathtt{collision\_ever}]$ & $0.969$ & $[0.926,0.998]$ & $0.002$ \\
Coins Own Coin First & \algo{IPPO-CER} & $1$ & $\Prob_{\ge0.4}[\neg w_1\ \Until\ m_1]$ & $0.503$ & $[0.460,0.546]$ & $0.044$ \\
Matrix Sustained Stag & \algo{Nash-Q-CER} & $1$ & $\Prob_{\ge0.9}[\Eventually(\mathtt{both\_stag}\land\Next(\mathtt{both\_stag}\land\Always\neg\mathtt{unmirrored}))]$ & $1.000$ & $[0.957,1.000]$ & $0.000$ \\
Commons Safe Harvest & \algo{IQL-CER} & $0$ & $\Prob_{\ge0.8}[\Always\neg\mathtt{player\_0\_zapped}\land\Eventually\mathtt{player\_0\_apple\_collected}]$ & $0.817$ & $[0.774,0.860]$ & $0.206$ \\
Territory Role Claim & \algo{IQL} & $0$ & $\Prob_{\ge0.25}[\neg\mathtt{p_1\_claimed}\ \Until\ (\mathtt{p_0\_claimed}\land\neg\mathtt{p_1\_claimed})]$ & $0.923$ & $[0.880,0.944]$ & $0.031$ \\
\bottomrule
\end{tabular}
\end{adjustbox}
\rowcolors{2}{}{}
\caption{Selected learners and agents (\texttt{player\_}$i$); objectives omit the outer $\Next$. Satisfaction estimates and marginal $95\%$ Hoeffding interval endpoints are five-seed means from $1000$ new rollouts per seed, without selection adjustment. Appendix~\ref{appsec:experiments} gives selection details, atom definitions and all agents' results.}
\label{tab:main_results}
\end{table*}

\paragraph{Counterfactual Experience Replay.}\label{sec:counterfactual_experience}
Terminal rewards are sparse and less suited to long-horizon tasks, so we use $\LTLf$ counterfactual experience replay (\algo{CER}), inspired by~\citet{voloshin2023eventual}. For each agent, we replay a recorded rollout chunk from alternative reward-automaton states, using the observed proposition labels to recompute monitor states and terminal rewards. Environment observations, actions, and episode endings stay fixed: the same rollout provides experience for several stages of progress towards the objective. \algo{IQL} variants store relabelled experience in replay, \algo{Nash-Q} updates online, and \algo{PPO} variants combine original and relabelled chunks and recompute rollout statistics. Appendix~\ref{appsec:cer} gives complete pseudocode and update details.

\paragraph{Environments.}
We consider 8 task instances from five environment families, fully specified in Appendix~\ref{appsec:environments} and Appendix~\ref{appsec:experiments}. The suite contains two- and three-player classic and capture variants of \env{Stochastic Pursuit Evasion}~\cite{weintraub2020introduction}, the repeated matrix-game \env{Stag Hunt}, and three spatial substrates adapted from the \env{Melting Pot} suite~\cite{agapiou2022melting}: \env{Coins}, \env{Commons Harvest}, and \env{Territory}. Each agent $i$ has objective $\Obj_i=\Prob_{\ge q_i}[\Next\PathObj_i]$ with a role-specific threshold $q_i$ (Appendix~\ref{appsec:experiments}).

\paragraph{Algorithms.}
Our setup targets five algorithm families: variants of \algo{IPPO}, \algo{PR2-IPPO}, \algo{IQL}, \algo{PR2-IQL}, and \algo{Nash-Q}, with and without counterfactual experience replay (\algo{CER}, e.g., \algo{IQL} and \algo{IQL-CER}), and one handcrafted-reward (\algo{hc-}) counterpart for each learner family. Probabilistic recursive reasoning (\algo{PR2}) evaluates actions against modelled opponent responses, guiding action selection and policy updates.

However, \algo{Nash-Q} completed only for \env{Stag Hunt} due to it being a tractable low-dimensional matrix game, whereas the other environments timed out after 96 hours. The handcrafted-reward baseline tries to capture the same objective faithfully with a Python reward function, but has a different equilibrium set in practice, and thus serves primarily as an auxiliary baseline.

\paragraph{Evaluation Metrics.}
Each fixed profile is evaluated on $1000$ episodes per seed, using five training seeds. We select the objective-trained learner and agent with the highest mean satisfaction in this initial evaluation. Table~\ref{tab:main_results} reports that agent's satisfaction estimate from an independent evaluation, threshold, marginal Hoeffding interval endpoints, and the profile's approximate NashConv. Appendix~\ref{appsec:experiments} gives the per-agent results, learning curves, and handcrafted comparisons.

We model the unilateral improvements found by finite-budget response searches using approximate NashConv~\cite{timbers2020approximate}. For each agent $i$, we freeze its opponents and continue training from the evaluated profile with the same learner and a fixed additional budget. In practice, this can instead be estimated using best-response predictors.

\begin{definition}[Approximate NashConv]\label{def:approx_nashconv}
Let $\StrategyProfile^{(i)}$ denote the profile obtained by the unilateral search for agent $i$. For objective-reward profiles, with payoff $u_i(\StrategyProfile)=\Pr^{\StrategyProfile}_{\mathcal M,s_0}(\Next\PathObj_i)$ and rollout estimate $\widehat u_i$, the score is
\[
\ApxNashConv(\StrategyProfile)\defeq
\sum_{i=1}^n\max\!\left\{0,\widehat u_i(\StrategyProfile^{(i)})-\widehat u_i(\StrategyProfile)\right\}.
\]
\end{definition}
The score in Definition~\ref{def:approx_nashconv} measures the unilateral improvements found by the search.  We call it approximate $\NashConv$ because each unilateral response is obtained by finite-budget training rather than exact best-response computation, and the resulting utility gains are estimated from finite rollouts. Consequently, finite search and sampling noise preclude an $\EpsNE$ certificate. As a heuristic, if $\ApxNashConv>0.05$, we train the original profile for up to $25\%$ additional steps. Unlike exploitability (Definition~\ref{def:exploitability}), which takes the maximum unilateral gain over agents, $\ApxNashConv$ sums the gains found across agents.

\paragraph{Results.}
Figure~\ref{fig:app_best_agent_learning} shows learning curves across all 8 benchmarks, while per-agent curves and complete learner comparisons appear in Appendix~\ref{appsec:experiments}. The selected learners in Table~\ref{tab:main_results} achieve independent satisfaction estimates above $0.9$ on six benchmarks: \algo{PR2-IPPO} on two-player Pursuit, \algo{PR2-IQL-CER} on Classic~3, \algo{IQL-CER} on Capture~3, \algo{IQL} on Territory, and \algo{Nash-Q-CER} on Stag Hunt. Harvest selects \algo{IQL-CER} with mean $0.817$, and Coins selects \algo{IPPO-CER} with mean $0.503$.

\paragraph{Threshold feasibility.} Capture~2 illustrates why rational synthesis is conditional on feasibility (Definition~\ref{def:rational_synthesis}). Its dynamics yield a numerically computed minimax guard satisfaction probability $v\approx0.130$. At any unrestricted $0.05$-NE, the guard's satisfaction probability is at most $v+0.05\approx0.180$, below its required $0.3$. The intruder's threshold $0.5$ remains feasible. Thus, the guard's acceptable-equilibrium set $\mathcal F_{0.05\text{-NE}}^{\mathrm g}$ is empty. Additional training is insufficient to resolve this, as the obstruction lies in the environment dynamics coupled with the specified requirement.

\section{Related Work}
\label{sec:related_work}
We provide an extended review of related work in Appendix~\ref{appsec:extended_literature_review}, but summarise the most relevant works here. \citeauthor{hahn2024multi} consider verifying minimax objectives using the logic \PATL, although they consider a model checking problem rather than a synthesis problem as they allow relabelling, and only allow an objective considered between a single coalition and its complement.  A closely related model-based tool is $\mathsf{PRISM}$-$\mathsf{Games}$~\cite{kwiatkowska2020prism}, which supports verification and strategy synthesis for $\rPATL$-style probabilistic and reward objectives in stochastic games, including multi-objective and equilibrium-based extensions~\cite{kwiatkowska2016model,kwiatkowska2013automated,kwiatkowska2020multi}. These methods assume explicit transition models.
\citeauthor{varricchione2023synthesising} construct a reward machine from an $\ATL$ specification, but their approach assumes a fully cooperative grand coalition, which in turn allows delegating to reward machines generated for individual agents. \citeauthor{hu2023reinforcement} consider reward machines for stochastic games, but are limited to exact Nash equilibria, non-quantitative specifications, and two-player general-sum games.  \citeauthor{jothimurugan2022specification} provide a bespoke enumerate-and-verify procedure for a non-probabilistic task specification language that is designed exclusively for equilibria with high social welfare. \citeauthor{hammond2021multi} studied temporal-logic reward optimisation of a sum of probabilities, but restricted their analysis to cooperative MARL. Our use of equilibrium constraints is inspired by classical rational synthesis~\cite{fisman2010rational,ummels2006rational,condurache2016complexity}. \citeauthor{subramanian2026multiagent} study reward-machine based learning in general-sum stochastic games, proposing CROM and QROM, which combine counterfactual experience updates over reward-machine states with opponent modelling and provide tabular convergence and finite-time guarantees for Nash Q-values under additional assumptions. However, our paper deals with probabilistic thresholds rather than discounted rewards. We provide sampling-based certification guarantees and bounds on discounting effects, as well as equilibrium stability and convergence analysis.

\section{Conclusions and Future Work}
\label{sec:conclusion}
Our approach targets multiple strategic agents, each equipped with its own probabilistic temporal-logic objective in $\PLTLf$ over finite traces. Each agent seeks a strategy satisfying its own probabilistic threshold within a profile that is stable under a chosen equilibrium notion (solution concept), interpreted over the utilities induced by the probabilistic requirements.

Compared to existing approaches, our work focuses on model-free strategy synthesis in a learning setting, without requiring known transition dynamics. This allows considering objectives beyond those that a strategic coalition can unilaterally enforce, as required by logics such as $\ATL$ and $\PATL$. We analyse how discounting affects the learning of probabilistic objectives, establish conditional convergence guarantees for tabular Nash-Q learning, and provide algorithm-agnostic finite-sample certificates of objective satisfaction.

Future work could target nested probabilistic properties and the undiscounted infinite-horizon case for $\omega$-regular properties, which would require learning maximal end components~\cite{chatterjee2011faster,hahn2019omega}, although this is non-trivial as end components for multi-agent systems depend on the dynamics of other agents. Another extension of interest could be evaluating different solution concepts, such as a jointly satisficing objective~\cite{rajasekaran2023multi}. It may also be worth investigating confidence-aware equilibrium selection with \algo{Nash-Q}; however this introduces statistical noise into the selector and non-local dependence across states.

\begin{acks}
The research described in this paper was supported by the EPSRC (grant number EP/X015823/1) and the UKRI Centre for Doctoral Training in Safe and Trusted AI (grant number EP/S0233356/1).
\end{acks}

\clearpage
\nobalance
\bibliographystyle{ACM-Reference-Format} 
\bibliography{cite}

\clearpage

\appendix
\raggedbottom %

\section*{Appendix}
The appendix is organised as follows:
\begin{itemize} [Appendix]
    \item \ref{appsec:pltlf_semantics}: The finite-trace semantics of $\PLTLf$
    \item \ref{appsec:overview}: Overview of our approach
    \item \ref{appsec:environments}: Environment descriptions and handcrafted rewards
    \item \ref{appsec:technical_details}: Algorithm details and proofs
    \item \ref{appsec:experiments}: Experimental results
    \item \ref{appsec:extended_literature_review}: Extended literature review
    \item \ref{appsec:hyperparams}: Hyperparameters
\end{itemize}

\section{\texorpdfstring{$\PLTLf$}{PLTLf} Semantics}\label{appsec:pltlf_semantics}

\begin{definition}[$\PLTLf$ Semantics]\label{def:pltlf-semantics}
Fix a labelled stochastic game $\mathcal M$ and a horizon of $H$ transitions (label traces have at most $H+1$ positions).
Let $\Pr^{\StrategyProfile}_{\mathcal M,s}$ be the probability measure induced by strategy profile $\StrategyProfile$ on complete finite paths
$\FPath^{\mathrm{comp}}_{\mathcal M,s}$.
For any finite trajectory $\pi=s_0 a_0 s_1\cdots s_\ell\in\FPath^{\mathrm{comp}}_{\mathcal M,s}$, let
$\trace(\pi)=\mathcal L(s_0)\cdots \mathcal L(s_\ell)\in\Alphabet^{\ell+1}$.
\end{definition}

\logicsemanticssetlhswidth{0.29\linewidth}
\logicsemanticssetrhswidth{0.63\linewidth}
\logicsemanticssetbodyfont{\footnotesize}
\begin{logicsemantics}[]
\logicsemanticsfirstsection{$\LTLf$ satisfaction (finite label traces $\trace\in\Alphabet^+$):}
\logicsemanticsrule{$\trace,i \models \top$}{}{}
\logicsemanticsrule{$\trace,i \models p$}{$\iff$}{$p\in \trace[i]$}
\logicsemanticsrule{$\trace,i \models \neg\varphi$}{$\iff$}{$\trace,i \not\models \varphi$}
\logicsemanticsrule{$\trace,i \models \varphi_1\land\varphi_2$}{$\iff$}{$(\trace,i\models\varphi_1)\ \wedge\ (\trace,i\models\varphi_2)$}
\logicsemanticsrule{$\trace,i \models \Next \varphi$}{$\iff$}{$i+1<|\trace|\ \wedge\ \trace,i+1\models\varphi$}
\logicsemanticsruleblock{$\trace,i \models \varphi_1\Until \varphi_2$}{$\iff$}{$\exists j\in[i,|\trace|-1].\ \big(\trace,j\models\varphi_2\ \wedge$\\$\forall k\in[i,j-1],\ \trace,k\models\varphi_1\big)$}
\logicsemanticssection{Path satisfaction:}
\logicsemanticsrule{$\mathcal M,\pi\models \varphi$}{$\iff$}{$\trace(\pi),0\models \varphi$}
\logicsemanticssection{Probabilistic modality:}
\logicsemanticsruleblock{$\mathcal M,s,\StrategyProfile \models \Prob_{\bowtie q}[\varphi]$}{$\iff$}{$\Pr^{\StrategyProfile}_{\mathcal M,s}\big(\{\pi\in\FPath^{\mathrm{comp}}_{\mathcal M,s}\mid$\\$\mathcal M,\pi\models \varphi\}\big)\ \bowtie q.$}
\end{logicsemantics}

\section{Overview}\label{appsec:overview}

Figure~\ref{fig:app_overview_pipeline} summarises our approach. Given each agent's objective and a target solution concept, we compile the path formulas into reward automata and form the product stochastic game. We train a single joint profile using model-free MARL, with optional counterfactual replay, and evaluate its satisfaction probabilities and \ApxNashConv.

\begin{figure*}[t]
  \centering
  \begin{adjustbox}{max width=\textwidth}
\pgfdeclarelayer{bg}
\pgfsetlayers{bg,main}
\begin{tikzpicture}[
  >=Latex,
  line width=0.75pt,
  node/.style={font=\footnotesize},
  box/.style={
    draw,
    rounded corners=1.5mm,
    fill=black!3,
    minimum height=1.7cm,
    text width=3.25cm,
    align=center,
    inner sep=4pt
  },
  decision/.style={
    diamond,
    draw,
    fill=black!3,
    aspect=2.15,
    text width=2.55cm,
    align=center,
    inner sep=1.5pt
  },
  arrow/.style={->, shorten >=2pt, shorten <=2pt},
  looplabel/.style={
    fill=white,
    inner sep=1.2pt,
    font=\footnotesize,
    align=center
  }
]
  \node[box] (objectives) at (0,0)
    {\textbf{Objective profile}\\Per-agent $\PLTLf$ objectives $\Prob_{\bowtie_i q_i}[\PathObj_i]$
    };
  \node[box] (automata) at (4.25,0)
    {\textbf{Reward automata}\\Compile each $\PathObj_i$ to a DFA and form the composite reward automaton $\mathcal A_{\PathObjVec}$};
  \node[box] (product) at (8.5,0)
    {\textbf{Product game}\\Synchronise $\mathcal M$ with automaton state and expose terminal acceptance rewards};

  \node[box] (cer) at (0,-2.9)
    {\textbf{Counterfactual replay}\\Relabel observed traces from alternative automaton states using the known $\delta_c$};
  \node[box] (learn) at (4.25,-2.9)
    {\textbf{Model-free learning}\\Train one joint profile with \algo{IPPO}, \algo{IQL}, \algo{NQL}, or \algo{PR2}};
  \node[box] (eval) at (8.5,-2.9)
    {\textbf{Assessment}\\Estimate satisfaction, Hoeffding bounds, and $\ApxNashConv$};

  \node[decision] (check) at (4.25,-5.65)
    {Empirical checks\\acceptable?};
  \node[box, minimum height=1.2cm] (report) at (8.9,-5.65)
    {\textbf{Candidate witness}\\Selected joint profile\\for reporting};

  \draw[arrow] (objectives) -- (automata);
  \draw[arrow] (automata) -- (product);
  \draw[arrow] (product.south) -- ++(0,-0.55) -| (learn.north);
  \draw[arrow] (learn.east) -- (eval.west);
  \draw[arrow] (eval.south) -- ++(0,-0.45) -| (check.north east);
  \draw[arrow] (check.east) --
    node[looplabel, above, pos=0.25] {yes}
    (report.west);

  \begin{pgfonlayer}{bg}
    \draw[arrow] (check.north) --
      node[looplabel, left, pos=0.38] {no}
      (learn.south);
    \node[looplabel, text width=2.35cm] at (3.0,-4.35)
      {continue training\\within budget};
  \end{pgfonlayer}
\end{tikzpicture}%
\end{adjustbox}
  \caption{Overview of our approach to automata-based rational synthesis.}
  \Description{Temporal objectives are compiled into reward automata and combined with the game. Model-free learners, optionally using counterfactual replay, produce a profile evaluated for satisfaction and strategic stability.}
    \label{fig:app_overview_pipeline}
\end{figure*}

\section{Environments}\label{appsec:environments}
We consider 8 benchmark instances drawn from five environment families. Table~\ref{tab:app_env_overview} summarises the complete suite.

\begin{table*}[t]
\centering
\small
\setlength{\tabcolsep}{5pt}
\renewcommand{\arraystretch}{1.15}
\rowcolors{2}{black!3}{white}
\begin{adjustbox}{max width=\textwidth,center}
\begin{tabular}{@{} L{0.17\textwidth} L{0.22\textwidth} C{0.08\textwidth} L{0.18\textwidth} L{0.27\textwidth} @{}}
\toprule
\textbf{Environment} & \textbf{Selected benchmark} & \textbf{Agents} & \textbf{Roles} & \textbf{High-level objective} \\
\midrule
\env{Stochastic Pursuit-Evasion} & Classic (2-player) & $2$ & Guard / intruder & Guard prevents a breach; the intruder reaches the goal without collision. \\
\env{Stochastic Pursuit-Evasion} & Capture (2-player) & $2$ & Guard / intruder & Guard seeks contact; the intruder avoids guard contact and all collisions. \\
\env{Stochastic Pursuit-Evasion} & Classic (3-player) & $3$ & Guard / two intruders & The opposed classic objectives are instantiated for both intruders. \\
\env{Stochastic Pursuit-Evasion} & Capture (3-player) & $3$ & Guard / two intruders & The opposed capture/stealth objectives are instantiated for both intruders. \\
\env{Coins} & Own Coin First (2-player) & $2$ & Two colour owners & Each player must collect its own colour before collecting the other colour. \\
\env{Matrix Stag Hunt} & Sustained Stag (2-player) & $2$ & Two symmetric players & Coordinate on a sustained mutual \(\mathtt{Stag}\) convention. \\
\env{Commons Harvest} & Safe Harvest (4-player) & $4$ & Four indexed harvesters & Each player collects an apple while avoiding being zapped. \\
\env{Territory} & Role Claim (2-player) & $2$ & Two competing claimants & Each player tries to make the first sole resource claim. \\
\bottomrule
\end{tabular}
\end{adjustbox}
\rowcolors{2}{}{}
\caption{Our eight environments and their high-level objectives and roles. Objectives that share a schema remain agent-indexed unless they are symmetric.}
\label{tab:app_env_overview}
\end{table*}

\begin{figure*}[t]
  \centering
  \subfloat[\env{Stochastic Pursuit-Evasion}]{
    \PngGraphic[0.21\linewidth]{0 0 540 540}{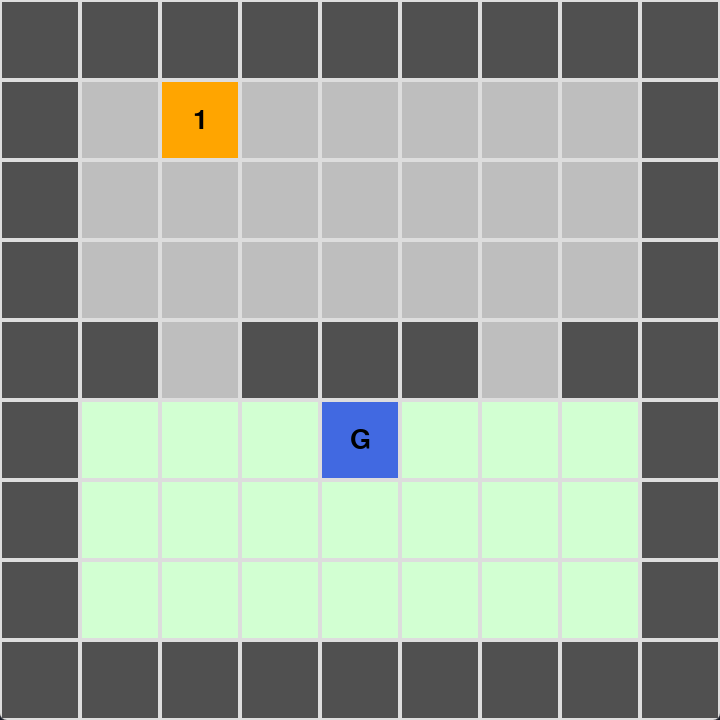}
  }
  \hfill
  \subfloat[\env{Coins}]{
    \PngGraphic[0.21\linewidth]{0 0 540 540}{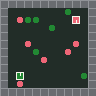}
  }
  \hfill
  \subfloat[\env{Territory}]{
    \PngGraphic[0.21\linewidth]{0 0 655 655}{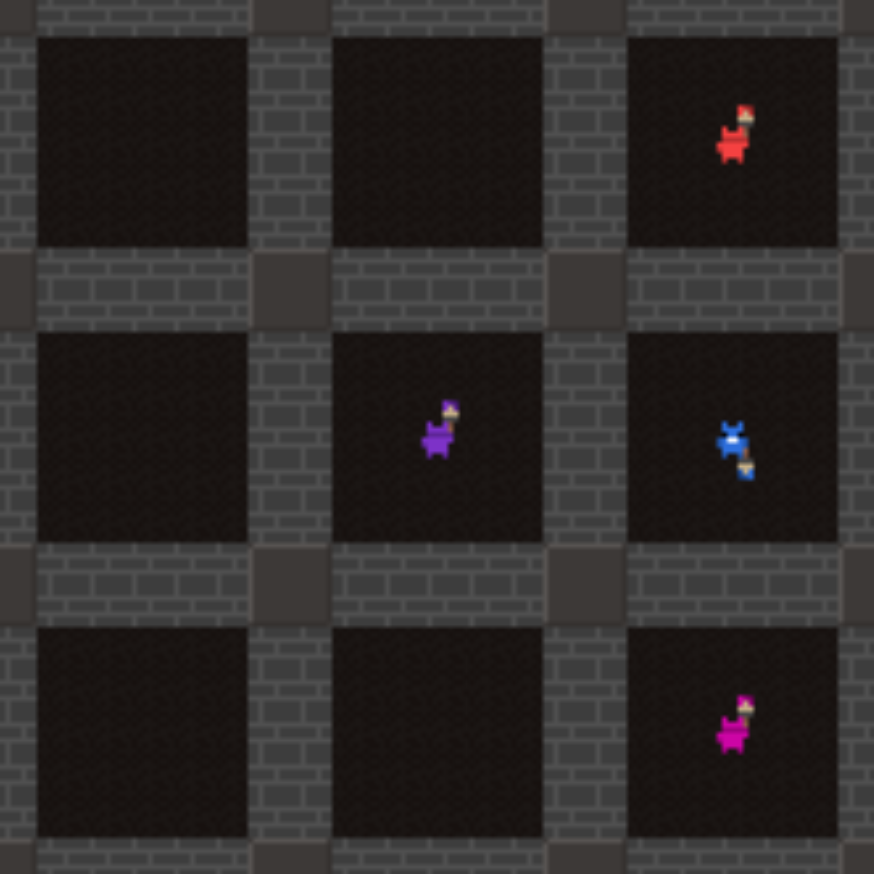}
  }
  \hfill
  \subfloat[\env{Commons Harvest}]{
    \PngGraphic[0.21\linewidth]{0 0 657 490}{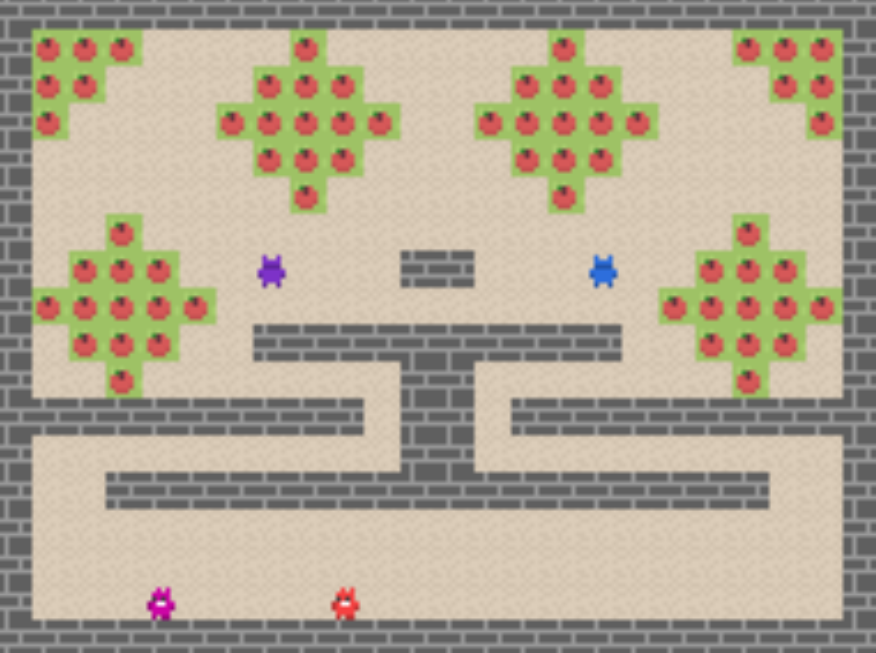}
  }
  \caption{The four spatial environment families used in our experiments.}
  \Description{Spatial layouts for Pursuit-Evasion, Coins, Territory, and Commons Harvest.}
  \label{fig:gridworld_visualisations}
\end{figure*}

Figure~\ref{fig:gridworld_visualisations} shows the spatial environments.

\paragraph{Observations.} Agents observe structured symbolic grids: each grid records, cell by cell, the identity of one layer of the scene (terrain, avatars and their orientations, coins, apples, claims, beams, etc.), alongside scalar side channels. The grids are flattened and scaled to $[0,1]$, and the index of the agent's reward-automaton state is appended alongside spatial coordinates.

\subsection{Stochastic Pursuit-Evasion}
We provide a stochastic variation of the \env{Pursuit Evasion} environment in the literature~\cite{weintraub2020introduction}, which is also referred to as \env{Guard Gridworld}. The slippery gridworld is instantiated with one guard, $\texttt{player}_0$, and either one or two intruders. In the classic task, the guard satisfies $\Always\neg\mathtt{goal\_reached}$, whereas every intruder satisfies $\Eventually\mathtt{goal\_reached}\land\Always\neg\mathtt{collision\_ever}$. In the capture task, the guard satisfies $\Eventually\mathtt{guard\_intruder\_collision}$, whereas each intruder satisfies $\Always\neg\mathtt{guard\_intruder\_collision}\land\Always\neg\mathtt{collision\_ever}$. Each move is subject to a role-specific slip probability, so the environment remains stochastic under a fixed joint policy.

\paragraph*{Handcrafted rewards.}
Let \(D_t\) denote episode completion after step \(t\), and let \(G_t\), \(C_t\), \(K_t\), and \(E_t\) denote the Boolean labels \(\mathtt{goal\_reached}\), \(\mathtt{collision}\), \(\mathtt{guard\_intruder\_collision}\), and \(\mathtt{collision\_ever}\), respectively. For the classic benchmark, the guard reward is
\[
\tilde r_{t,\mathrm{classic}}^0 =
\begin{cases}
-1, & G_t,\\
1, & D_t \land \neg G_t,\\
0, & \text{otherwise},
\end{cases}
\]
and each intruder \(i>0\) receives
\[
\tilde r_{t,\mathrm{classic}}^i = \mathbf{1}[G_t]-\mathbf{1}[C_t].
\]
For the capture benchmark, the guard reward is
\[
\tilde r_{t,\mathrm{capture}}^0 =
\begin{cases}
1, & K_t,\\
-0.5, & D_t \land \neg K_t,\\
0, & \text{otherwise},
\end{cases}
\]
and each intruder \(i>0\) receives
\[
\tilde r_{t,\mathrm{capture}}^i =
\begin{cases}
-1, & K_t \lor E_t,\\
1, & D_t \land \neg K_t \land \neg E_t,\\
0, & \text{otherwise}.
\end{cases}
\]

\subsection{Coins}
This benchmark uses a two-player adaptation of the \env{Melting Pot} \texttt{coins} substrate with a $256$-step horizon and a coin regrowth rate of $0.005$. Each player owns one coin colour. Writing $m_i$ and $w_i$ for the indexed propositions that player $i$ collected a matching or mismatching coin, respectively, its Own Coin First objective is
\[
\PathObj_i = \neg w_i\ \Until\ m_i,
\qquad i\in\{0,1\}.
\]
The strong-until objective requires a matching collection to occur and forbids a mismatching collection beforehand. Coin consumption and regrowth create strategic externalities even though the formulae are symmetric.

\paragraph*{Handcrafted reward.}
On each coin collection, the collector receives
\[
\tilde r_t^i = \mathbf 1[m_i]-\mathbf 1[w_i].
\]
The other agent receives no reward for that collection.

\subsection{Matrix Stag Hunt}
This benchmark uses a repeated two-player matrix-game version of the classic \env{Stag Hunt}. At each of $20$ stages, both agents choose between \(\mathtt{Hare}\) and \(\mathtt{Stag}\). The public state records the previous joint action and whether the agents chose different actions at that stage. The base environment truncates only at the horizon, and the default stage payoffs are shown in Figure~\ref{fig:stag_hunt_payoff}: mutual \(\mathtt{Stag}\) is payoff-dominant, mutual \(\mathtt{Hare}\) is the safer coordinated outcome, and unilateral \(\mathtt{Stag}\) receives the failed-hunt payoff.

\begin{figure}[tbp]
  \centering
  \begin{tikzpicture}[
    font=\sffamily,
    payoff/.style={minimum width=2.6cm, minimum height=1.6cm, align=center},
    label/.style={font=\sffamily\bfseries},
    grid/.style={draw=black!18, line width=0.6pt},
    box/.style={rounded corners=6pt, draw=black!55, line width=0.9pt}
]

\matrix (P) [
    matrix of nodes,
    nodes in empty cells,
    row sep=-\pgflinewidth,
    column sep=-\pgflinewidth,
]{
    |[payoff]| $3,3$ & |[payoff]| $3,0$ \\
    |[payoff]| $0,3$ & |[payoff]| $4,4$ \\
};

\begin{scope}[on background layer]
    \clip[rounded corners=6pt]
        (P-1-1.north west)
        rectangle
        (P-2-2.south east);

    \fill[blue!6]
        (P-1-1.north west) rectangle (P-1-1.south east);
    \fill[green!8]
        (P-2-2.north west) rectangle (P-2-2.south east);

    \fill[black!2]
        (P-1-2.north west) rectangle (P-1-2.south east);
    \fill[black!2]
        (P-2-1.north west) rectangle (P-2-1.south east);
\end{scope}

\draw[box]
    (P-1-1.north west)
    rectangle
    (P-2-2.south east);

\draw[grid] (P-1-1.south west) -- (P-1-2.south east);
\draw[grid] (P-1-1.north east) -- (P-2-1.south east);

\node[label] (B1) at ($ (P-1-1.north) + (0,11pt) $) {$\mathbf{Hare}$};
\node[label] (B2) at ($ (P-1-2.north) + (0,11pt) $) {$\mathbf{Stag}$};

\node[label] at ($ (B1.north)!0.5!(B2.north) + (0,11pt) $) {Player 2};

\node[label] at ($ (P-1-1.west) + (-13pt,0) $) {$\mathbf{Hare}$};
\node[label] at ($ (P-2-1.west) + (-13pt,0) $) {$\mathbf{Stag}$};

\node[label, rotate=90] at ($ (P-1-1.west)!0.5!(P-2-1.west) + (-36pt,0) $) {Player 1};

\end{tikzpicture}
  \caption{Stage payoff matrix for \env{Matrix Stag Hunt}. The ordered pair in each cell gives the payoff to Player 1 and Player 2, respectively.}
  \Description{Rows and columns are Hare and Stag for Players 1 and 2. The payoffs are (3,3), (3,0), (0,3), and (4,4), respectively.}
  \label{fig:stag_hunt_payoff}
\end{figure}

Both agents have the Sustained Stag objective and receive the same terminal Boolean reward for satisfying
\[
\begin{gathered}
\PathObj_i =
\Eventually\!\big(
  \mathtt{both\_stag}
  \land
  \Next(\mathtt{both\_stag}\land\Always\,\neg\mathtt{unmirrored})
\big),
\\
i\in\{0,1\}.
\end{gathered}
\]
The atoms \(\mathtt{both\_stag}\) and \(\mathtt{unmirrored}\) are derived from the previous stage outcome. Thus, the objective allows failed coordination attempts before the convention is established, but requires two consecutive \(\mathtt{Stag}\)-\(\mathtt{Stag}\) outcomes after which the agents choose the same action at every remaining stage. Earlier pairs followed by a mismatch are permitted, as are later coordinated \(\mathtt{Hare}\)-\(\mathtt{Hare}\) outcomes.

\paragraph*{Handcrafted reward.}
We use the native stage payoff as the handcrafted reward for \env{Stag Hunt}:
\[
\tilde r_t^i = r_{\mathrm{env},t}^i =
\begin{cases}
3, & a_t^0=a_t^1=\mathtt{Hare},\\
4, & a_t^0=a_t^1=\mathtt{Stag},\\
0, & a_t^i=\mathtt{Stag}\land a_t^{-i}=\mathtt{Hare},\\
3, & a_t^i=\mathtt{Hare}\land a_t^{-i}=\mathtt{Stag}.
\end{cases}
\]

\subsection{Commons Harvest}
This benchmark uses the Melting Pot substrate \texttt{commons\_\allowbreak harvest\_\_\allowbreak partnership}, where four harvesters compete for apples in a shared field and can zap one another. Each player has the \env{Safe Harvest} objective:
\begin{multline*}
\PathObj_i = \Always\neg\mathtt{player\_i\_zapped}
\land\Eventually\mathtt{player\_i\_apple\_collected},\\
i\in\{0,1,2,3\}.
\end{multline*}

\paragraph*{Handcrafted reward.}
Each agent receives $+1$ for the first apple it collects before being zapped. The first zap gives $-1$ and prevents further apple rewards, cancelling any earlier apple reward. Another player's apple collection or zap does not affect this reward. We obtain apple labels from the collector's positive native payoff and zap labels from the substrate's per-target zap events.

\subsection{Territory}
This benchmark uses the \texttt{territory\_\_rooms} substrate from Melting Pot. In the two-player \env{Role Claim} task, players race to claim a resource first. Let $c_i$ denote the ownership-change event \(\mathtt{player\_i\_resource\_claimed}\). The objectives are
\[
\PathObj_0 = \neg c_1\ \Until\ (c_0\land\neg c_1),
\qquad
\PathObj_1 = \neg c_0\ \Until\ (c_1\land\neg c_0).
\]
Thus each player tries to make the first sole claim. Simultaneous first claims are ties, and later claims cannot reverse the temporal outcome.

\paragraph*{Handcrafted reward.}
At the first step on which a claim occurs, the sole claimant receives $+1$ and the opponent $-1$. A simultaneous first claim gives both players $0$. Later claims give no reward.

\subsection{Handcrafted Baselines}\label{appsec:handcrafted_algorithms}
We train each handcrafted baseline (\algo{hc-IPPO}, \algo{hc-PR2-IPPO}, \algo{hc-IQL}, \algo{hc-PR2-IQL} and \algo{hc-Nash-Q}) on the rewards above, using the corresponding learner's hyperparameters (Appendix~\ref{appsec:hyperparams}) and five seeds $\{0,\dots,4\}$, without counterfactual replay. The $\PLTLf$ monitors record objective satisfaction throughout training, while the learner receives the handcrafted reward.

We evaluate each fixed profile on $1000$ new episodes per seed, reporting objective satisfaction and Hoeffding intervals at overall $\delta=0.05$, as for the objective-trained profiles. Each learning-curve panel marks the strongest handcrafted baseline's final satisfaction with a diamond, using the same aggregation over agents as the curves. We place the diamond at the last training episode across that baseline's seeds, capped at the end of the plotted curves. We also compute its \ApxNashConv from improvements in temporal satisfaction (Appendix~\ref{appsec:deviation_search}).

\section{Technical Details}\label{appsec:technical_details}

\subsection{Counterfactual Experience Replay}\label{appsec:cer}
We apply CER separately to each agent's deterministic reward automaton. We append the automaton-state identifier to the agent's observation and store the successor-state propositions at each step. Relabelling keeps the environment rollout fixed and recomputes the automaton states from alternative start states. We keep the original sample and discard duplicate relabellings. Our procedure adapts~\citet{voloshin2023eventual} to finite-trace $\LTLf$. In particular, our deterministic finite automata require no $\varepsilon$-jump transitions used with B\"uchi automata in $\LTL$, and rewards depend on acceptance at episode termination instead of repeated visits to B\"uchi accepting states.

Write $x_t=(o_t,m_t)$ for the stored observation, with $o_t$ the environment features and $m_t$ the appended monitor identifier. Relabelling replaces that identifier, giving $\tilde x_t=(o_t,\tilde m_t)$. In Algorithm~\ref{alg:cer_relabel}, $d_t$ denotes episode completion under the environment termination/truncation rule or wrapper horizon, independently of the relabelled monitor state. A rollout-chunk boundary alone does not set $d_t=1$. Acceptance is evaluated at the post-transition monitor state $m^+$ before reset, while the stored next observation at an episode boundary uses the reset monitor state.

\paragraph*{Actor-critic integration.}
\algo{IPPO-CER} and \algo{PR2-IPPO-CER} combine factual and relabelled chunks, recomputing rollout statistics with fixed pre-update parameters. The implementation retains factual behaviour log-probabilities for conditional importance correction. Write $\rho_t=\pi_i^{\mathrm{old}}(a_t^i\mid\tilde x_t)/\mu_i(a_t^i\mid x_t)$, where $\mu_i$ generated the factual action. In \texttt{corrected} mode, $\min(1,\rho_t)$ weights the clipped PPO surrogate separately from PPO's update-ratio clipping. IPPO also uses $\lambda$-V-trace-style targets~\cite{espeholt2018impala}, with $\lambda=0.95$ and action and trace caps one. PR2-IPPO retains its response-aware advantage and one-step joint-action critic. In \texttt{critic\_only} mode, actor updates, entropy and advantage normalisation use factual samples, while relabelled samples update the critic. Each run records its replay mode and importance-cap settings. Evaluation always uses the original temporal objectives.

\begin{algorithm}[t]
\caption{Relabelling a rollout chunk for one agent}
\label{alg:cer_relabel}
\KwIn{agent reward automaton \(A_i=(M_i,m^0_i,\Alphabet,\delta_i,F_i)\), rollout chunk \((x_0,a_0,r_0,\ell_0,d_0,\allowbreak\dots,x_{T-1},a_{T-1},r_{T-1},\ell_{T-1},d_{T-1},x_T)\), where \(x_t=(o_t,m_t)\) and \(\ell_t=\mathcal L(s_{t+1})\)}
\KwOut{set \(\widetilde{\mathcal C}\) of relabelled chunks}
Initialise \(\widetilde{\mathcal C}\) with the factual chunk\;
\For{each counterfactual start state \(\tilde m_0 \in M_i\)}{
    Set \(m \leftarrow \tilde m_0\)\;
    \For{\(t=0,\dots,T-1\)}{
        Set relabelled observation \(\tilde x_t\leftarrow(o_t,m)\)\;
        Compute pre-reset successor \(m^+ \leftarrow \delta_i(m,\ell_t)\)\;
        Set terminal reward \(\tilde r_t \leftarrow \mathbf{1}[d_t]f_i(\mathbf{1}[m^+\in F_i])\)\;
        Set \(m\leftarrow\tilde m_0\) if \(d_t\), else \(m^+\)\;
        Set stored next observation \(\tilde x_{t+1}\leftarrow(o_{t+1},m)\)\;
    }
    Add \((\tilde x_0,a_0,\tilde r_0,d_0,\dots,\tilde x_T)\) to \(\widetilde{\mathcal C}\) if non-duplicate\;
}
\Return{\(\widetilde{\mathcal C}\)}
\end{algorithm}

\begin{algorithm}[t]
\caption{CER integration for replay-based learners}
\label{alg:cer_q}
\KwIn{replay buffer \(\mathcal D\), factual rollout chunk \(c\)}
Compute \(\widetilde{\mathcal C}\leftarrow\) \textsc{RelabelChunk}\((c)\)\;
\For{each \(\tilde c \in \widetilde{\mathcal C}\)}{
    Insert \(\tilde c\) into \(\mathcal D\)\;
}
\Return{\(\mathcal D\)}
\end{algorithm}

\begin{algorithm}[t]
\caption{CER integration for actor-critic learners}
\label{alg:cer_ac}
\KwIn{factual rollout batch \(\mathcal B\), fixed pre-update parameters}
\For{each agent \(i\)}{
    \For{each factual rollout chunk \(c\) of agent \(i\)}{
        Compute \(\widetilde{\mathcal C}\leftarrow\) \textsc{RelabelChunk}\((c)\)\;
        Add all chunks in \(\widetilde{\mathcal C}\) to agent \(i\)'s update batch\;
    }
    Recompute the required rollout statistics using the fixed pre-update parameters\;
}
\Return{updated batch \(\mathcal B\)}
\end{algorithm}

\paragraph*{Nash-Q integration}
For \algo{NQL}, we apply CER online: we relabel each observed joint transition and perform a tabular Nash-Q backup on the original transition and each relabelled variant before the next environment step.

\subsection{\algo{PR2}} \label{appsec:PR2}
Our discrete-action variants of probabilistic recursive reasoning (\algo{PR2}) are inspired by~\citet{wen2019pr2} and use the \algo{IQL} and \algo{IPPO} training interfaces. For each agent \(i\), we encode the opponents' actions as a single joint action \(a^{-i}\in\mathcal A^{-i}=\prod_{j\ne i}\mathcal A^j\). The agent has a joint-action critic
\[
Q_i(x_i)\in\mathbb R^{|\mathcal A^i|\times |\mathcal A^{-i}|},
\]
where \(x_i\) is the agent observation, including the reward-automaton state when monitoring is enabled. The opponent-response distribution is induced directly from this critic,
\[
\rho_i(a^{-i}\mid x_i,a^i)
=
\operatorname{softmax}_{a^{-i}} Q_i(x_i,a^i,a^{-i}),
\]
and the value of the agent's action under this response model is
\[
\bar Q_i(x_i,a^i)
=
\sum_{a^{-i}\in\mathcal A^{-i}}
\rho_i(a^{-i}\mid x_i,a^i) Q_i(x_i,a^i,a^{-i}).
\]

We use unit temperature for $\rho_i$. This response model comes from agent $i$'s own critic and need not match the opponents' conditional action distribution. The PR2-IPPO advantage below uses this model.

\paragraph*{\algo{PR2-IQL}}
Action selection is \(\varepsilon\)-greedy with respect to \(\bar Q_i\). Replay stores the agent's own action, the observed encoded opponent joint action, and the terminal automaton reward. The update regresses only the observed joint-action entry:
\[
\begin{aligned}
\mathcal L_i=\Big(&Q_i(x_i,a^i,a^{-i}_{\mathrm{obs}})-r_i\\
&-\gamma(1-d)\max_{b^i\in\mathcal A^i}\bar Q_i^{\mathrm{targ}}(x'_i,b^i)\Big)^2.
\end{aligned}
\]

\paragraph*{\algo{PR2-IPPO}}
The actor outputs logits for \(\pi_i(\cdot\mid x_i)\), while the critic has the same joint-action shape as above. The state value used by the policy update is
\[
V_i(x_i)=\sum_{a^i\in\mathcal A^i}\pi_i(a^i\mid x_i)\bar Q_i(x_i,a^i),
\]
and the rollout advantage, fixed before optimisation, is
\[
\begin{aligned}
\widehat A_i^\rho(x_i,a^i)=\operatorname{stopgrad}\!\Big[&\bar Q_i^{\mathrm{old}}(x_i,a^i)\\
&-\sum_b\pi_i^{\mathrm{old}}(b\mid x_i)\bar Q_i^{\mathrm{old}}(x_i,b)\Big].
\end{aligned}
\]
Here ``old'' denotes the pre-update parameters. The clipped PPO objective uses this advantage after normalisation over each minibatch, while the critic regresses
\[
Q_i(x_i,a^i,a^{-i}_{\mathrm{obs}})
\quad\text{towards}\quad
r_i+\gamma(1-d)V_i^{\mathrm{old}}(x'_i).
\]
The \texttt{-CER} variants use the relabelling procedure in Appendix~\ref{appsec:cer}. For \algo{PR2-IPPO}, the log-probabilities and advantages are recomputed after relabelling.

\subsection{Unilateral Improvement Search}\label{appsec:deviation_search}
For each agent, we initialise its corresponding learner (\algo{IPPO}, \algo{PR2-IPPO}, \algo{IQL}, \algo{PR2-IQL}, or \algo{Nash-Q}) from the evaluated profile and train that agent with its opponents fixed. The additional budget is the product of the BR iteration and batch-step counts in Table~\ref{tab:hyperparams_eval}. Response searches preserve the learner's CER settings. For \texttt{hc-} profiles, the deviating agent instead trains on its temporal objective with CER (Appendix~\ref{appsec:handcrafted_algorithms}).

We evaluate the candidate and original policies with paired episode seeds and measure their temporal satisfaction $u_i$, including for handcrafted profiles. The score sums the unilateral improvements found within the given budget and learner family.

\paragraph{Frozen-opponent operator.}
The Nash-Q trainer has a dedicated single-agent training mode. It obtains opponents' policies from a frozen copy of the original joint Q-tables and replaces the Nash continuation value by $\max_{a_i}\mathbb E_{a_{-i}\sim\sigma_{-i}}Q_i(x',a_i,a_{-i})$. Candidate evaluation likewise samples opponents from that frozen copy. With full product state and the temporal utility, the corresponding undiscounted recursion would be
\[
B_i(x)=\max_{a_i}\sum_{a_{-i}}\sigma_{-i}(a_{-i}\mid x)\sum_{x'}\tilde{\mathcal P}(x'\mid x,a_i,a_{-i})B_i(x'),
\]
with $B_i(x)=R_i(x)$ at termination and sums and maxima over legal actions. In our experiments, the search uses the learner's discount factor, while the score measures improvements in the undiscounted temporal utility $u_i$.

\paragraph{Aggregation and sampling.}
Writing $\Delta_i=\sup_{\sigma_i'}u_i(s_0,(\sigma_i',\sigma_{-i}))-u_i(s_0,\StrategyProfile)$, exact NashConv is $\sum_i\Delta_i$, whereas exploitability is $\max_i\Delta_i$. Retaining the original strategy is admissible, so every $\Delta_i\ge0$ and hence $\Exploit\le\sum_i\Delta_i\le n\Exploit$. Two gains of $0.03$ therefore give NashConv $0.06$ while satisfying $0.05$-NE. Finite search can miss deviations, while sampling noise and clipping can inflate either aggregation, so the estimated score is not a statistical lower bound on exploitability.

For paired success differences in $[-1,1]$, a simultaneous two-sided Hoeffding half-width over $m$ agents is $\sqrt{2\ln(2m/\delta)/N}$. At $N=32$ and $\delta=0.05$, it is approximately $0.52$ for two agents and $0.56$ for four, substantially exceeding $\varepsilon=0.05$. Small scores therefore do not establish strategic stability.

\subsection{Proof of Theorem~\ref{thm:coherence}}\label{appsec:proof_coherence}
\coherencethm*

\begin{proof}
By construction, $\mathcal R_\varphi(\pi)=1$ iff $\mathcal M,\pi\models \varphi$, and otherwise $\mathcal R_\varphi(\pi)=0$. Therefore
\[
\mathcal R_\varphi(\pi)=\mathbf{1}\big[\mathcal M,\pi\models \varphi\big].
\]
Taking expectation under $\Pr^{\StrategyProfile}_{\mathcal M,s}$ yields
\[
\begin{aligned}
\mathbb E_{\pi\sim\Pr^{\StrategyProfile}_{\mathcal M,s}}\!\left[\mathcal R_\varphi(\pi)\right]
&=
\mathbb E_{\pi\sim\Pr^{\StrategyProfile}_{\mathcal M,s}}\!\left[\mathbf{1}\big[\mathcal M,\pi\models \varphi\big]\right]\\
&=
\Pr^{\StrategyProfile}_{\mathcal M,s}\!\big(\{\pi\mid \mathcal M,\pi\models \varphi\}\big)
=
\Pr^{\StrategyProfile}_{\mathcal M,s}(\varphi).
\end{aligned}
\]
\end{proof}

\Needspace{10\baselineskip}
\subsection{Proof of Theorem~\ref{thm:markov_sufficiency_product}}\label{appsec:proof_markov_sufficiency}

\markovsufficiencythm*

We adapt the equivalence argument of~\cite{brafman2018ltlf,bacchus1996rewarding} to stochastic games.

\begin{definition}[Non-Markovian Reward Stochastic Game]\label{def:nmrsg}
A \emph{non-Markovian reward stochastic game} (NMRSG) $\mathcal N$ inherits the states, initial state, legal actions, transition kernel, horizon $H$, and stopping rule $\kappa$ of $\mathcal M$. Rewards are zero before termination, with undiscounted terminal payoffs $\mathcal R^i(h_T)$ depending on the complete history $h_T=s_0\bm a_0\cdots s_T$.
\end{definition}

\begin{definition}[Equivalence of Reward Games]\label{def:equivalence_nmrsg}
An NMRSG $\mathcal N$ is \emph{equivalent} to an extended Markovian game $\widehat{\mathcal G}$ with zero non-terminal rewards and undiscounted terminal payoffs $\widehat R_i$ if there are maps $\rho:\widehat{\mathcal S}\to\mathcal S$ and $\iota:\mathcal S\to\widehat{\mathcal S}$ such that:
\begin{enumerate}
\item $\rho(\iota(s))=s$ for every $s$, and agent $i$'s legal actions at each non-terminal $x$ are $d^i(\rho(x))$.
\item For every non-terminal $x$ and legal joint action $\bm a$, $\rho$ maps positive-probability extended successors bijectively onto original successors, preserving probabilities:
\[
\widehat{\mathcal P}(x'\mid x,\bm a)=\mathcal P(\rho(x')\mid\rho(x),\bm a).
\]
\item For corresponding feasible histories $h_t=s_0\bm a_0\cdots s_t$ and $\hat h_t=x_0\bm a_0\cdots x_t$, with $x_0=\iota(s_0)$ and $\rho(x_k)=s_k$, $x_t$ is terminal iff $\kappa(s_t,t)=1$, and $\widehat R_i(x_T)=\mathcal R^i(h_T)$ for every agent $i$ at termination.
\end{enumerate}
\end{definition}
In Definition~\ref{def:equivalence_nmrsg}, $\iota$ initialises the additional state and $\rho$ projects back to the environment, preserving dynamics, termination, and payoffs.

\begin{lemma}[Strategy-Profile Equivalence]\label{lemma:markovian_equivalence}
For equivalent games, feasible histories from $s_0$ lift uniquely from $\iota(s_0)$. Memoryless extended strategies correspond to original strategies depending only on the final lifted state, up to choices at unreachable states. With independent randomisation across agents, corresponding profiles preserve complete-history probabilities, expected terminal payoffs, and unilateral deviation gains within these strategy classes.
\end{lemma}
\begin{proof}
Condition~2 constructs the lift recursively, with projection as its inverse, while conditions~1 and~3 preserve legal actions and stopping. For corresponding histories ending at $x_t$, set $\sigma_i(h_t)=\hat\sigma_i(x_t)$. This defines both directions of the correspondence, since original strategies depend only on $x_t$. Choose arbitrary legal distributions at unreachable non-terminal states. Independent randomisation gives equal joint-action probabilities, so induction using condition~2 gives equal history probabilities. Condition~3 then gives equal expected payoffs. Since the correspondence acts agent by agent, it also preserves unilateral deviations and their gains.
\end{proof}

\begin{proof}[Proof of Theorem~\ref{thm:markov_sufficiency_product}]
Consider the NMRSG of Definition~\ref{def:nmrsg} with terminal payoffs
\[
\mathcal R^i(h_T)\defeq f_i\big(\mathbf 1[\mathcal M,h_T\models\PathObj_i]\big).
\]
Their expectations equal $u_i(s_0,\StrategyProfile)$ because $f_i$ is affine. To establish equivalence with $\widehat{\mathcal M}_H$, define
\[
\begin{aligned}
\iota(s)&\defeq\big(s,\delta_c(m_0^c,\mathcal L(s)),0\big),\\
\rho(s,\bm m,t)&\defeq s.
\end{aligned}
\]
Condition~1 is immediate. By Definition~\ref{def:product_stochastic_game}, each original successor $s'$ gives the unique product successor
\[
\big(s',\delta_c(\bm m,\mathcal L(s')),t+1\big)
\]
with probability $\mathcal P(s'\mid s,\bm a)$, establishing condition~2. Both games terminate at the same $T$ under $\kappa$. The monitor reads exactly the labels through $s_T$, so
\[
\begin{aligned}
R_i(s_T,\bm m_T,T)
&=f_i\big(\mathbf 1[\bm m_T\models F_i]\big)\\
&=f_i\big(\mathbf 1[\mathcal M,h_T\models\PathObj_i]\big)\\
&=\mathcal R^i(h_T),
\end{aligned}
\]
which establishes condition~3.

The final lifted state is $(s_t,\bm m_t,t)$, so the original strategy class in Lemma~\ref{lemma:markovian_equivalence} is exactly $\Sigma$. The lemma therefore gives both directions of the memoryless correspondence and preserves utilities. Acceptance indicators also agree on corresponding complete histories, so equality of their probabilities gives equality of every satisfaction probability $p_i$.
\end{proof}

\subsection{Proof of Theorem~\ref{thm:terminal_invariance}}\label{appsec:proof_terminal_invariance}
\terminalinvariancethm*

\begin{proof}
Since $r_t=0$ for all $t<H$ and $r_H=\mathbf 1[\varphi]$, we have
\[
\sum_{t=0}^{H}\gamma^t r_t=\gamma^H r_H=\gamma^H\mathbf 1[\varphi].
\]
Taking expectation under $\StrategyProfile$ gives
\[
\mathbb{E}_{\StrategyProfile}\!\left[\sum_{t=0}^{H}\gamma^t r_t\right]
=
\gamma^H\mathbb{E}_{\StrategyProfile}[\mathbf 1[\varphi]]
=
\gamma^H\Pr_{\StrategyProfile}(\varphi).
\]
Because $\gamma^H>0$ for $\gamma\in(0,1)$ and fixed $H$, multiplying by the constant $\gamma^H$ does not change the set of maximisers. Hence, any strategy profile maximising discounted return also maximises $\Pr_{\StrategyProfile}(\varphi)$. The same positive scaling applies to every unilateral gain, and commutes with their supremum and maximum, giving the stated equilibrium consequences.
\end{proof}

\subsection{Proof of Proposition~\ref{prop:early_term_discount}}\label{appsec:proof_early_term_discount}
\earlytermdiscountprop*

\begin{proof}
By assumption, $r_t=0$ for all $t<T$ and $r_T=\mathbf 1[A]$. Therefore
\[
G_\gamma=\sum_{t=0}^{T}\gamma^t r_t
=
\sum_{t=0}^{T-1}\gamma^t\cdot 0+\gamma^T r_T
=
\gamma^T\mathbf 1[A].
\]
Taking expectation under the trajectory distribution induced by $\StrategyProfile$ yields
\[
\mathbb E_{\StrategyProfile}[G_\gamma]
=
\mathbb E_{\StrategyProfile}\!\big[\gamma^T\mathbf 1[A]\big].
\]
\end{proof}

\subsection{Proof of Corollary~\ref{cor:early_term_bounds}}\label{appsec:proof_early_term_bounds}
\earlytermboundscor*

\begin{proof}
Under the assumptions of Proposition~\ref{prop:early_term_discount}, on the event $A$ we have $T\le H$, and therefore $\gamma^T\in[\gamma^H,1]$. On the complement event $\neg A$, the return is $0$. Hence, pointwise,
\[
\gamma^H\mathbf 1[A]\le \gamma^T\mathbf 1[A]\le \mathbf 1[A].
\]
Taking expectation under $\StrategyProfile$ and using Proposition~\ref{prop:early_term_discount} gives
\[
\gamma^H\Pr_{\StrategyProfile}(A)
\le
\mathbb E_{\StrategyProfile}[G_\gamma]
\le
\Pr_{\StrategyProfile}(A).
\]
\end{proof}

\subsection{Proof of Theorem~\ref{thm:hoeffding_ci}}\label{appsec:proof_hoeffding_ci}
\hoeffdingcithm*

\begin{proof}
The variables $X_1,\dots,X_n$ are i.i.d.\ and take values in $[0,1]$ with mean $p$. By Hoeffding's inequality, for any $t>0$,
\[
\Pr(\hat p-p\ge t)\le e^{-2nt^2}
\qquad\text{and}\qquad
\Pr(p-\hat p\ge t)\le e^{-2nt^2}.
\]
Applying a union bound yields
\[
\Pr(|\hat p-p|\ge t)\le 2e^{-2nt^2}.
\]
Now set
\[
t=\rho_n(\delta)=\sqrt{\frac{\ln(2/\delta)}{2n}}.
\]
Then $2e^{-2nt^2}=\delta$, so
\[
\Pr\big(|\hat p-p|\le \rho_n(\delta)\big)\ge 1-\delta.
\]
Whenever $|\hat p-p|\le \rho_n(\delta)$ holds, we also have
\[
p\in [\hat p-\rho_n(\delta),\hat p+\rho_n(\delta)].
\]
Since $p\in[0,1]$, clipping the endpoints to $[0,1]$ preserves this containment, so
\[
\Pr\!\big(p\in[L_n,U_n]\big)\ge 1-\delta.
\]
\end{proof}

\subsection{Proof of Corollary~\ref{cor:gamma_epsilon_discount_bias}} \label{appsec:proof_gamma_epsilon_discount_bias}
\gammaepsilondiscountbiascor*

\begin{proof}
By Corollary~\ref{cor:early_term_bounds},
\[
0
\le
\Pr_{\StrategyProfile}(A)-\mathbb E_{\StrategyProfile}[G_\gamma]
\le
(1-\gamma^H)\Pr_{\StrategyProfile}(A)
\le
(1-\gamma^H),
\]
since $\Pr_{\StrategyProfile}(A)\le 1$. If $\gamma\ge(1-\xi)^{1/H}$, then $\gamma^H\ge 1-\xi$, and therefore $1-\gamma^H\le\xi$. Combining the inequalities gives the result.
\end{proof}

The bound is tight in that if $A$ occurs with probability $1$ and \(T=H\) almost surely, then
$
\Pr_{\StrategyProfile}(A)-\mathbb E_{\StrategyProfile}[G_\gamma]
=
1-\gamma^H.
$
Hence no smaller \(\gamma\) can guarantee bias at most \(\xi\) uniformly (over all possible acceptance-time distributions).

\subsection{Proof of Corollary~\ref{cor:threshold_cert}}\label{appsec:proof_threshold_cert}
\thresholdcertcor*

\begin{proof}
Let $E=\{p\in[L_n,U_n]\}$. By Theorem~\ref{thm:hoeffding_ci}, we have $\Pr(E)\ge 1-\delta$. On the event $E$, each claim follows directly from the appropriate endpoint comparison:
\[
L_n>q \;\Rightarrow\; p\ge L_n>q,
\]
\[
U_n<q \;\Rightarrow\; p\le U_n<q,
\]
\[
L_n\ge q \;\Rightarrow\; p\ge L_n\ge q,
\]
and
\[
U_n\le q \;\Rightarrow\; p\le U_n\le q.
\]
These are exactly the semantic conditions for $\Prob_{>q}[\varphi]$, $\Prob_{<q}[\varphi]$, $\Prob_{\ge q}[\varphi]$, and $\Prob_{\le q}[\varphi]$, respectively.
\end{proof}

\subsection{Proof of Corollary~\ref{cor:joint_confidence_cert}}\label{appsec:proof_joint_confidence_cert}
\jointconfidencecertcor*

\begin{proof}
For each $i$, let $E_i=\{p_i\in[L_i,U_i]\}$. By
Theorem~\ref{thm:hoeffding_ci}, $\Pr(E_i^c)\le\delta_i$. Hence the
Bonferroni inequality (union bound) gives
\[
\begin{aligned}
\Pr\!\left(\bigcap_{i=1}^m E_i\right)
&=1-\Pr\!\left(\bigcup_{i=1}^m E_i^c\right)\\
&\ge 1-\sum_{i=1}^m\Pr(E_i^c)\\
&\ge 1-\sum_{i=1}^m\delta_i\ge 1-\delta.
\end{aligned}
\]
No independence among the events $E_i$ is required. On their intersection,
each endpoint comparison implies the corresponding threshold constraint by
Corollary~\ref{cor:threshold_cert}, so all such implications hold
simultaneously. If every objective is certified, this is precisely
$\mathcal M,s,\StrategyProfile\models\ObjVec$. Finally, substituting
$\delta_i=\delta/m$ into $\rho_{n_i}(\delta_i)$ yields the stated
uniform half-width.
\end{proof}

\subsection{Proof of Proposition~\ref{prp:spne_finite_horizon}}\label{appsec:proof_spne_finite_horizon}
\spnefinitehorizonprop*

\begin{proof}
This is the standard backward-induction argument~\cite{fudenberg1991game}, recorded for completeness. Let $r_i$ be agent $i$'s stage reward and set $V_i^t(s)$ to agent $i$'s terminal payoff whenever $(s,t)$ is terminal, including at $t=H$. In $\widehat{\mathcal M}_H$, $r_i\equiv 0$, $\gamma=1$, and the terminal payoff is $R_i(x)$. For $t=H-1,\dots,0$ and every non-terminal $(s,t)$, let $G_t(s)$ be the finite normal-form game with action sets $d^1(s),\dots,d^n(s)$ and payoffs
\[
g_i^t(s,\bm a)\ \defeq \ r_i(s,\bm a)+\gamma\sum_{s'}\mathcal P(s'\mid s,\bm a)\,V_i^{t+1}(s').
\]
By Nash's theorem~\cite{Nash1950}, $G_t(s)$ has a mixed equilibrium $\alpha^t(s)$. Set $\sigma_i(s,t)\defeq \alpha_i^t(s)$ and $V_i^t(s)\defeq \mathbb E_{\bm a\sim\alpha^t(s)}\big[g_i^t(s,\bm a)\big]$. By induction on $H-t$, $V_i^t(s)$ is agent $i$'s expected payoff under $\StrategyProfile$ from state $s$ at time $t$, irrespective of the preceding history. Hence, as $\alpha^t(s)$ is an equilibrium of $G_t(s)$, no agent gains from a one-stage deviation at any history. By the one-stage deviation principle for finite-horizon games~\cite{fudenberg1991game}, no deviation is profitable in any subgame, so $\StrategyProfile$ is an $\SPNE$, and it is Markov by construction.
\end{proof}

\subsection{Proof of Theorem~\ref{thm:nashq_convergence}}\label{appsec:proof_nashq_convergence}

\nashqconvergencethm*

\begin{proof}
\emph{Equilibrium.} Definition~\ref{def:selector_profile} applies the backward induction of Proposition~\ref{prp:spne_finite_horizon} to $\widehat{\mathcal M}_H$, choosing $\Sel_x(Q^{\Sel}(x))$ at each non-terminal state. Hence $\StrategyProfile_{\Sel}$ is a Markov $\SPNE$ with continuation payoffs $V^{\Sel}$.

\emph{Convergence of Q-values.} By (i), convex updates and equilibrium averaging keep all Q-values and targets in $[0,1]^n$. We prove convergence of $Q^k$ and $V^k$ on reachable states by backward induction on time. At every terminal state, $V^k=R=V^{\Sel}$, covering all states at $t=H$ and any earlier stopping states. Suppose $V^k(y)\to V^{\Sel}(y)$ almost surely at every reachable state $y$ at time $t+1$. Let $\mathcal G_k$ contain the past and the current choice of updated pair and step sizes, before sampling its successor $x'$. For an updated pair $(x,\bm a)$ at a non-terminal state in time layer $t$, condition (iii) and the definition of $Q^{\Sel}$ give
\[
\begin{aligned}
\left\|\mathbb E[V^k(x')\mid\mathcal G_k]-Q^{\Sel}(x,\bm a)\right\|_\infty&\le c_k,\\
c_k\defeq\max_y\left\|V^k(y)-V^{\Sel}(y)\right\|_\infty&\longrightarrow0,
\end{aligned}
\]
where the maximum is over reachable states at time $t+1$. The target errors $V_i^k(x')-Q_i^{\Sel}(x,\bm a)$ are bounded by one. Subtracting $Q_i^{\Sel}(x,\bm a)$ from the update gives the scalar recursion of Lemma~1 of~\cite{singh2000convergence} coordinatewise, with $\alpha_k=0$ when $(x,\bm a)$ is not updated. Thus (ii) and that lemma, with $\kappa=0$ and vanishing bias $c_k$, imply $Q^k(x)\to Q^{\Sel}(x)$ almost surely for every reachable non-terminal $x$ at time $t$. Condition (iv) then gives $V^k(x)=\Val_x(Q^k(x))\to V^{\Sel}(x)$, completing the induction. Finiteness makes these conclusions simultaneous on all reachable states.

\emph{Sufficient conditions for (iv).} Continuity of $\Sel_x$ at $g^\star=Q^{\Sel}(x)$ implies continuity of $\Val_x(g)=\mathbb E_{\bm a\sim\Sel_x(g)}[g(\bm a)]$ there. Alternatively, suppose all equilibria of $g^\star$ have value $v^\star$. For any $g_j\to g^\star$, compactness and the equilibrium inequalities ensure that every limit point of $\Sel_x(g_j)$ is an equilibrium of $g^\star$. Every corresponding value limit is therefore $v^\star=\Val_x(g^\star)$, proving continuity.

\emph{Continuity of payoffs and exploitability.} On Markov profiles, $u_i(s_0,\StrategyProfile)$ is a polynomial in the action probabilities at reachable states. Against fixed Markov opponents, agent $i$ faces a finite-horizon MDP, so its best-response payoff over all deviations is attained by a deterministic Markov policy~\cite{puterman2014markov}. It is therefore the maximum of finitely many such polynomials. Consequently, both $u_i(s_0,\cdot)$ and $\Exploit(s_0,\cdot)$ are continuous on Markov profiles.

\emph{Limiting profiles.} Fix an outcome on which $Q^k\to Q^{\Sel}$. By compactness, every subsequence of $(\StrategyProfile_k)$ has a further subsequence converging on reachable non-terminal states to some $\bar{\StrategyProfile}$. Write its indices as $k_\ell$. Passing to the limit in the stage-equilibrium inequalities shows that $\bar{\StrategyProfile}(x)$ is an equilibrium of $Q^{\Sel}(x)$. Moreover, along this subsequence,
\[
\begin{aligned}
\mathbb E_{\bm a\sim\bar{\StrategyProfile}(x)}[Q^{\Sel}(x,\bm a)]
&=\lim_{\ell\to\infty}\Val_x(Q^{k_\ell}(x))\\
&=V^{\Sel}(x),
\end{aligned}
\]
where the first equality follows from convergence of both the profiles and Q-values, and the second from (iv). Backward induction then gives continuation payoffs $V^{\Sel}$ and equilibrium in every reachable subgame, as in Proposition~\ref{prp:spne_finite_horizon}. Thus $u_i(s_0,\bar{\StrategyProfile})=u_i(s_0,\StrategyProfile_{\Sel})$ and $\Exploit(s_0,\bar{\StrategyProfile})=0$. Continuity yields the claimed payoff and exploitability convergence, and hence eventual $\EpsNE$ for each $\varepsilon>0$. Finally, continuity of each $\Sel_x$ at $Q^{\Sel}(x)$ gives $\StrategyProfile_k(x)\to\StrategyProfile_{\Sel}(x)$ directly.
\end{proof}

\section{Experiments}\label{appsec:experiments}
\subsection{Experimental Setup}\label{appsec:selection_protocol}
For each task, we select the learner, discount and agent with the highest mean satisfaction in the initial evaluation, using five seeds and the training budgets in Table~\ref{tab:hyperparams_protocol}. Appendix~\ref{appsec:hyperparams} summarises training discounts by environment and lists the remaining hyperparameters.

\paragraph{Learning curves.} Each row shows one agent, with the best variant per learner family on the left and all objective-trained variants on the right. Variants are chosen by final smoothed training satisfaction averaged over five seeds, also in Figure~\ref{fig:app_best_agent_learning}. Solid and dashed lines distinguish variants without and with CER; diamonds mark the strongest handcrafted baseline. Means and $95\%$ confidence bands use exponential smoothing ($\alpha=0.002$), followed by a centred window covering $3\%$ of the episode range. Bands may extend beyond the vertical axis limits.

\paragraph{Computational resources.} Each algorithm--task job uses seeds $\{0,1,2,3,4\}$, one physical CPU core (two threads) on an AMD EPYC 7552, 7.81\,GiB RAM and a 96-hour limit, without a GPU. Nash-Q variants complete only on Matrix Stag Hunt, using one Intel Core Ultra 7 155H thread per job and the same training and evaluation budgets.

\begin{table*}[t]
\centering
\small
\setlength{\tabcolsep}{3pt}
\renewcommand{\arraystretch}{1.15}
\rowcolors{2}{black!3}{white}
\begin{adjustbox}{max width=\textwidth,center}
\begin{tabular}{@{} L{0.15\textwidth} C{0.045\textwidth} C{0.08\textwidth} L{0.415\textwidth} C{0.045\textwidth} C{0.1\textwidth} C{0.05\textwidth} C{0.081\textwidth} @{}}
\toprule
\textbf{Benchmark} &
\textbf{$n$} &
\textbf{Selected agent} &
\textbf{Selected agent's objective} &
\textbf{$\widehat p_i$} &
\textbf{95\% CI} &
\textbf{$\underline m$} &
\textbf{$\ApxNashConv$} \\
\midrule
Pursuit Classic 2 & $2$ & \texttt{player\_0} & $\Prob_{\ge0.1}[\Always\neg\mathtt{goal\_reached}]$ & $0.979$ & $[0.936,0.993]$ & $-0.744$ & $0.000$ \\
Pursuit Capture 2 & $2$ & \texttt{player\_1} & $\Prob_{\ge0.5}[\Always\neg\mathtt{guard\_intruder\_collision}\land\Always\neg\mathtt{collision\_ever}]$ & $1.000$ & $[0.957,1.000]$ & $-0.300$ & $0.000$ \\
Pursuit Classic 3 & $3$ & \texttt{player\_2} & $\Prob_{\ge0.1}[\Eventually\mathtt{goal\_reached}\land\Always\neg\mathtt{collision\_ever}]$ & $0.975$ & $[0.932,0.993]$ & $-0.500$ & $0.008$ \\
Pursuit Capture 3 & $3$ & \texttt{player\_2} & $\Prob_{\ge0.8}[\Always\neg\mathtt{guard\_intruder\_collision}\land\Always\neg\mathtt{collision\_ever}]$ & $0.969$ & $[0.926,0.998]$ & $-0.049$ & $0.002$ \\
Coins Own Coin First & $2$ & \texttt{player\_1} & $\Prob_{\ge0.4}[\neg w_1\ \Until\ m_1]$ & $0.503$ & $[0.460,0.546]$ & $+0.046$ & $0.044$ \\
Matrix Sustained Stag & $2$ & \texttt{player\_1} & $\Prob_{\ge0.9}[\Eventually(\mathtt{both\_stag}\land\Next(\mathtt{both\_stag}\land\Always\neg\mathtt{unmirrored}))]$ & $1.000$ & $[0.957,1.000]$ & $+0.053$ & $0.000$ \\
Commons Safe Harvest & $4$ & \texttt{player\_0} & $\Prob_{\ge0.8}[\Always\neg\mathtt{player\_0\_zapped}\land\Eventually\mathtt{player\_0\_apple\_collected}]$ & $0.817$ & $[0.774,0.860]$ & $-0.383$ & $0.206$ \\
Territory Role Claim & $2$ & \texttt{player\_0} & $\Prob_{\ge0.25}[\neg\mathtt{p_1\_claimed}\ \Until\ (\mathtt{p_0\_claimed}\land\neg\mathtt{p_1\_claimed})]$ & $0.923$ & $[0.880,0.944]$ & $-0.156$ & $0.031$ \\
\bottomrule
\end{tabular}
\end{adjustbox}
\rowcolors{2}{}{}
\caption{Selected-agent objectives, marginal intervals and joint margins for the learners in Table~\ref{tab:main_results}. Satisfaction estimates and marginal $95\%$ Hoeffding interval endpoints are five-seed means from $1000$ independent confirmatory rollouts per seed. The reported margin averages $\underline m(\StrategyProfile)\defeq\min_i(L_i-q_i)$ over seeds, where $L_i$ is the joint Hoeffding lower bound with $\delta_i=0.05/n$ for $n$ agents. A non-negative per-seed margin certifies every agent's threshold for that profile. \ApxNashConv is averaged over seeds. Objectives omit the outer $\Next$.}
\label{tab:app_best_agent_summary}
\end{table*}

\begin{table*}[t]
\centering
\footnotesize
\setlength{\tabcolsep}{3pt}
\renewcommand{\arraystretch}{1.12}
\rowcolors{2}{black!3}{white}
\begin{adjustbox}{max width=\textwidth,center}
\begin{tabular}{@{} L{0.09\textwidth} L{0.08\textwidth} L{0.33\textwidth} C{0.04\textwidth} C{0.06\textwidth} C{0.11\textwidth} C{0.11\textwidth} C{0.06\textwidth} C{0.08\textwidth} @{}}
\toprule
\textbf{Benchmark} & \textbf{Agent} & \textbf{Objective $\PathObj_i$} & \textbf{$q_i$} & \textbf{$\widehat p_i$} & \textbf{Marginal 95\% CI} & \textbf{Joint 95\% CI} & \makecell[t]{\textbf{Certified}\\\textbf{seeds}} & \textbf{$\ApxNashConv$} \\
\midrule
Classic 2 & \texttt{player\_0} & $\Always\neg\mathtt{goal\_reached}$ & $0.1$ & $0.979$ & $[0.936,0.993]$ & $[0.932,0.994]$ & $5/5$ & $0.000$ \\
Classic 2 & \texttt{player\_1} & $\Eventually\mathtt{goal\_reached}\land\Always\neg\mathtt{collision\_ever}$ & $0.75$ & $0.020$ & $[0.007,0.063]$ & $[0.006,0.067]$ & $0/5$ & $0.000$ \\
\addlinespace
Capture 2 & \texttt{player\_0} & $\Eventually\mathtt{guard\_intruder\_collision}$ & $0.3$ & $0.000$ & $[0.000,0.043]$ & $[0.000,0.047]$ & $0/5$ & $0.000$ \\
Capture 2 & \texttt{player\_1} & $\Always\neg\mathtt{guard\_intruder\_collision}\land\Always\neg\mathtt{collision\_ever}$ & $0.5$ & $1.000$ & $[0.957,1.000]$ & $[0.953,1.000]$ & $5/5$ & $0.000$ \\
\addlinespace
Classic 3 & \texttt{player\_0} & $\Always\neg\mathtt{goal\_reached}$ & $0.5$ & $0.011$ & $[0.000,0.054]$ & $[0.000,0.060]$ & $0/5$ & $0.008$ \\
Classic 3 & \texttt{player\_1} & $\Eventually\mathtt{goal\_reached}\land\Always\neg\mathtt{collision\_ever}$ & $0.1$ & $0.975$ & $[0.932,0.993]$ & $[0.926,0.994]$ & $5/5$ & $0.008$ \\
Classic 3 & \texttt{player\_2} & $\Eventually\mathtt{goal\_reached}\land\Always\neg\mathtt{collision\_ever}$ & $0.1$ & $0.975$ & $[0.932,0.993]$ & $[0.926,0.994]$ & $5/5$ & $0.008$ \\
\addlinespace
Capture 3 & \texttt{player\_0} & $\Eventually\mathtt{guard\_intruder\_collision}$ & $0.05$ & $0.030$ & $[0.002,0.073]$ & $[0.001,0.079]$ & $0/5$ & $0.002$ \\
Capture 3 & \texttt{player\_1} & $\Always\neg\mathtt{guard\_intruder\_collision}\land\Always\neg\mathtt{collision\_ever}$ & $0.8$ & $0.969$ & $[0.926,0.998]$ & $[0.920,0.999]$ & $5/5$ & $0.002$ \\
Capture 3 & \texttt{player\_2} & $\Always\neg\mathtt{guard\_intruder\_collision}\land\Always\neg\mathtt{collision\_ever}$ & $0.8$ & $0.969$ & $[0.926,0.998]$ & $[0.920,0.999]$ & $5/5$ & $0.002$ \\
\bottomrule
\end{tabular}
\end{adjustbox}
\rowcolors{2}{}{}
\caption{Per-agent Pursuit results: \algo{PR2-IPPO} for Classic~2 and Capture~2, \algo{PR2-IQL-CER} for Classic~3, and \algo{IQL-CER} for Capture~3. Satisfaction estimates and Hoeffding interval endpoints are five-seed means from $1000$ independent confirmatory rollouts per seed, shared across agents. Marginal intervals use $\delta=0.05$; joint intervals use $\delta_i=0.05/m$ for $m$ agents. \emph{Certified seeds} counts runs whose joint lower bound $L_i$ meets the row's threshold $q_i$. \ApxNashConv is averaged over seeds and shared across each profile's agents.}
\label{tab:app_results_pursuit}
\end{table*}

\PursuitAgentLearningFigure{Classic}{figures/experiments/pursuit/classic_2/summary}{classic_2}

\PursuitAgentLearningFigure{Capture}{figures/experiments/pursuit/capture_2/summary}{capture_2}

\PursuitThreeAgentLearningFigure{Classic}{figures/experiments/pursuit/classic_3}{classic_3}

\PursuitThreeAgentLearningFigure{Capture}{figures/experiments/pursuit/capture_3}{capture_3}

\begin{table*}[t]
\centering
\footnotesize
\setlength{\tabcolsep}{3pt}
\renewcommand{\arraystretch}{1.12}
\rowcolors{2}{black!3}{white}
\begin{adjustbox}{max width=\textwidth,center}
\begin{tabular}{@{} L{0.10\textwidth} L{0.37\textwidth} C{0.05\textwidth} C{0.07\textwidth} C{0.12\textwidth} C{0.12\textwidth} C{0.06\textwidth} C{0.08\textwidth} @{}}
\toprule
\textbf{Agent} & \textbf{Objective $\PathObj_i$} & \textbf{$q_i$} & \textbf{$\widehat p_i$} & \textbf{Marginal 95\% CI} & \textbf{Joint 95\% CI} & \makecell[t]{\textbf{Certified}\\\textbf{seeds}} & \textbf{$\ApxNashConv$} \\
\midrule
\texttt{player\_0} & $\neg w_0\ \Until\ m_0$ & $0.4$ & $0.507$ & $[0.464,0.550]$ & $[0.460,0.554]$ & $5/5$ & $0.044$ \\
\texttt{player\_1} & $\neg w_1\ \Until\ m_1$ & $0.4$ & $0.503$ & $[0.460,0.546]$ & $[0.456,0.550]$ & $5/5$ & $0.044$ \\
\bottomrule
\end{tabular}
\end{adjustbox}
\rowcolors{2}{}{}
\caption{Per-agent Coins results for \algo{IPPO-CER}; $m_i,w_i$ denote matching- and mismatching-coin collections. Satisfaction estimates and Hoeffding interval endpoints are five-seed means from $1000$ independent confirmatory rollouts per seed, shared across agents. Marginal intervals use $\delta=0.05$; joint intervals use $\delta_i=0.025$ for the two agents. \emph{Certified seeds} counts runs whose joint lower bound $L_i$ meets the row's threshold $q_i$. \ApxNashConv is averaged over seeds and shared across agents.}
\label{tab:app_results_coins_full}
\end{table*}

\CoinsAgentLearningFigure{figures/experiments/coins/own_coin_first_2_full}

\begin{table*}[t]
\centering
\footnotesize
\setlength{\tabcolsep}{3pt}
\renewcommand{\arraystretch}{1.12}
\rowcolors{2}{black!3}{white}
\begin{adjustbox}{max width=\textwidth,center}
\begin{tabular}{@{} L{0.10\textwidth} L{0.37\textwidth} C{0.05\textwidth} C{0.07\textwidth} C{0.12\textwidth} C{0.12\textwidth} C{0.06\textwidth} C{0.08\textwidth} @{}}
\toprule
\textbf{Agent} & \textbf{Objective $\PathObj_i$} & \textbf{$q_i$} & \textbf{$\widehat p_i$} & \textbf{Marginal 95\% CI} & \textbf{Joint 95\% CI} & \makecell[t]{\textbf{Certified}\\\textbf{seeds}} & \textbf{$\ApxNashConv$} \\
\midrule
\texttt{player\_0} & $\Eventually(\mathtt{both\_stag}\land\Next(\mathtt{both\_stag}\land\Always\neg\mathtt{unmirrored}))$ & $0.9$ & $1.000$ & $[0.957,1.000]$ & $[0.953,1.000]$ & $5/5$ & $0.000$ \\
\texttt{player\_1} & $\Eventually(\mathtt{both\_stag}\land\Next(\mathtt{both\_stag}\land\Always\neg\mathtt{unmirrored}))$ & $0.9$ & $1.000$ & $[0.957,1.000]$ & $[0.953,1.000]$ & $5/5$ & $0.000$ \\
\bottomrule
\end{tabular}
\end{adjustbox}
\rowcolors{2}{}{}
\caption{Per-agent Matrix Stag Hunt results for \algo{Nash-Q-CER}. Satisfaction estimates and Hoeffding interval endpoints are five-seed means from $1000$ independent confirmatory rollouts per seed, shared across agents. Marginal intervals use $\delta=0.05$; joint intervals use $\delta_i=0.025$ for the two agents. \emph{Certified seeds} counts runs whose joint lower bound $L_i$ meets the row's threshold $q_i$. \ApxNashConv is averaged over seeds and shared across agents.}
\label{tab:app_results_stag_hunt_matrix}
\end{table*}

\StagHuntAgentLearningFigure{figures/experiments/stag_hunt/sustained_stag_2/summary}

\begin{table*}[t]
\centering
\footnotesize
\setlength{\tabcolsep}{3pt}
\renewcommand{\arraystretch}{1.12}
\rowcolors{2}{black!3}{white}
\begin{adjustbox}{max width=\textwidth,center}
\begin{tabular}{@{} L{0.10\textwidth} L{0.37\textwidth} C{0.05\textwidth} C{0.07\textwidth} C{0.12\textwidth} C{0.12\textwidth} C{0.06\textwidth} C{0.08\textwidth} @{}}
\toprule
\textbf{Agent} & \textbf{Objective $\PathObj_i$} & \textbf{$q_i$} & \textbf{$\widehat p_i$} & \textbf{Marginal 95\% CI} & \textbf{Joint 95\% CI} & \makecell[t]{\textbf{Certified}\\\textbf{seeds}} & \textbf{$\ApxNashConv$} \\
\midrule
\texttt{player\_0} & $\Always\neg\mathtt{player\_0\_zapped}\land\Eventually\mathtt{player\_0\_apple\_collected}$ & $0.8$ & $0.817$ & $[0.774,0.860]$ & $[0.767,0.867]$ & $3/5$ & $0.206$ \\
\texttt{player\_1} & $\Always\neg\mathtt{player\_1\_zapped}\land\Eventually\mathtt{player\_1\_apple\_collected}$ & $0.7$ & $0.608$ & $[0.565,0.651]$ & $[0.558,0.658]$ & $3/5$ & $0.206$ \\
\texttt{player\_2} & $\Always\neg\mathtt{player\_2\_zapped}\land\Eventually\mathtt{player\_2\_apple\_collected}$ & $0.3$ & $0.001$ & $[0.000,0.044]$ & $[0.000,0.051]$ & $0/5$ & $0.206$ \\
\texttt{player\_3} & $\Always\neg\mathtt{player\_3\_zapped}\land\Eventually\mathtt{player\_3\_apple\_collected}$ & $0.2$ & $0.005$ & $[0.000,0.048]$ & $[0.000,0.056]$ & $0/5$ & $0.206$ \\
\bottomrule
\end{tabular}
\end{adjustbox}
\rowcolors{2}{}{}
\caption{Per-agent Commons Harvest results for \algo{IQL-CER}. Satisfaction estimates and Hoeffding interval endpoints are five-seed means from $1000$ independent confirmatory rollouts per seed, shared across agents. Marginal intervals use $\delta=0.05$; joint intervals use $\delta_i=0.0125$ for the four agents. \emph{Certified seeds} counts runs whose joint lower bound $L_i$ meets the row's threshold $q_i$. \ApxNashConv is averaged over seeds and shared across agents.}
\label{tab:app_results_commons}
\end{table*}

\HarvestAgentLearningFigure{figures/experiments/harvest/safe_harvest_4/summary}

\begin{table*}[t]
\centering
\footnotesize
\setlength{\tabcolsep}{3pt}
\renewcommand{\arraystretch}{1.12}
\rowcolors{2}{black!3}{white}
\begin{adjustbox}{max width=\textwidth,center}
\begin{tabular}{@{} L{0.10\textwidth} L{0.37\textwidth} C{0.05\textwidth} C{0.07\textwidth} C{0.12\textwidth} C{0.12\textwidth} C{0.06\textwidth} C{0.08\textwidth} @{}}
\toprule
\textbf{Agent} & \textbf{Objective $\PathObj_i$} & \textbf{$q_i$} & \textbf{$\widehat p_i$} & \textbf{Marginal 95\% CI} & \textbf{Joint 95\% CI} & \makecell[t]{\textbf{Certified}\\\textbf{seeds}} & \textbf{$\ApxNashConv$} \\
\midrule
\texttt{player\_0} & $\neg\mathtt{p_1\_claimed}\ \Until\ (\mathtt{p_0\_claimed}\land\neg\mathtt{p_1\_claimed})$ & $0.25$ & $0.923$ & $[0.880,0.944]$ & $[0.876,0.945]$ & $5/5$ & $0.031$ \\
\texttt{player\_1} & $\neg\mathtt{p_0\_claimed}\ \Until\ (\mathtt{p_1\_claimed}\land\neg\mathtt{p_0\_claimed})$ & $0.2$ & $0.054$ & $[0.045,0.097]$ & $[0.044,0.101]$ & $1/5$ & $0.031$ \\
\bottomrule
\end{tabular}
\end{adjustbox}
\rowcolors{2}{}{}
\caption{Per-agent Territory results for \algo{IQL}; $\mathtt{p_i\_claimed}$ abbreviates $\mathtt{player\_i\_resource\_claimed}$. Satisfaction estimates and Hoeffding interval endpoints are five-seed means from $1000$ independent confirmatory rollouts per seed, shared across agents. Marginal intervals use $\delta=0.05$; joint intervals use $\delta_i=0.025$ for the two agents. \emph{Certified seeds} counts runs whose joint lower bound $L_i$ meets the row's threshold $q_i$. \ApxNashConv is averaged over seeds and shared across agents.}
\label{tab:app_results_territory}
\end{table*}

\TerritoryAgentLearningFigure{figures/experiments/territory/role_claim_2/summary}

\clearpage

\section{Extended Literature Review}\label{appsec:extended_literature_review}
\paragraph*{Probabilistic Strategic Logics and Synthesis Tools.}
\citet{hahn2024multi} translate reachability-normal $\PATL$ objectives into rewards for Minimax Q-Learning. Their model-checking procedure recursively relabels states for nested modalities, rather than finding one strategy witnessing all subformulae. It considers a single coalition in turn-based games. \citet{bozkurt2021model} similarly study model-free learning for $\LTL$ objectives in two-player turn-based zero-sum games, using deterministic Rabin automata.

The zero-sum fragments of $\mathsf{PRISM}$-$\mathsf{Games}$~\cite{kwiatkowska2016model} synthesise strategies for probabilistic and reward-based $\rPATL$ objectives from known transition models. Single- and multi-objective properties concern one coalition against its complement; where optimal strategies need not exist~\cite{chen2013stochastic}, the tool computes $\varepsilon$-approximations of Pareto-optimal strategies.

Equilibrium-based extensions synthesise subgame-perfect, social-welfare-optimal Nash equilibria in concurrent stochastic games~\cite{kwiatkowska2013automated,kwiatkowska2020multi,kwiatkowska2020prism}, also assuming known models.

\citet{varricchione2023synthesising} construct reward machines from $\ATL$ specifications for a fully cooperative grand coalition.

\paragraph*{Rational Synthesis.}
Rational synthesis~\cite{fisman2010rational,ummels2006rational} considers agents pursuing their own temporal goals. \citet{fisman2010rational} reduce $\LTL$ rational synthesis to satisfiability in Extended Strategy Logic, obtaining $\TWOEXPTIME$-completeness for dominant-strategy, Nash and subgame-perfect solution concepts.

\citet{ummels2006rational} studies pure strategies in infinite multiplayer turn-based games with $\omega$-regular objectives, establishing existence, payoff-constrained decision procedures and finite-state implementations for subgame-perfect Nash equilibria. \citet{condurache2016complexity} study complexity refinements and cooperative and non-cooperative variants.

Our setting combines per-agent probabilistic $\LTLf$ requirements with learning from unknown dynamics under a single committed profile. Policy Space Response Oracles~\cite{lanctot2017unified} instead train approximate best responses against a meta-strategy over a policy population.

\paragraph*{Rational Verification.}
Rational verification~\cite{wooldridge2016rational} asks whether a temporal property holds on some or all equilibrium behaviours. Equilibrium Logic expresses such properties through modalities $[\NE]$ and $\langle\NE\rangle$ over Nash-equilibrium computations. Model checking is $\TWOEXPTIME$-hard in general, with matching upper bounds for restricted fragments.

\paragraph*{Commitment-aware Logics.}
Strategy Logic ($\SL$) explicitly quantifies and binds strategies, subsumes $\ATL^\star$, and can express equilibrium through stability against deviations. Model checking for full $\SL$ is decidable but non-elementary.

Extensions of $\ATL$ evaluate nested formulas under previously selected strategies, including strategy contexts ($\ATLsc$)~\cite{da2010atl}, explicit strategies ($\ATLes$)~\cite{walther2007alternating}, and irrevocable semantics~\cite{aagotnes2007alternating}.

\section{Hyperparameters}\label{appsec:hyperparams}

\subsection{Training Budgets}

Tables~\ref{tab:hyperparams_protocol}--\ref{tab:hyperparams_eval} give the settings for all tasks and learners, each using seeds $\{0,1,2,3,4\}$.

\paragraph{Training discounts.} The environment-level discount summary is \(\gamma=0.99\) for all \env{Pursuit-Evasion} variants and \env{Matrix Stag Hunt}, and \(\gamma=1\) for \env{Coins}, \env{Commons Harvest}, and \env{Territory}. Evaluation uses undiscounted satisfaction probabilities.

\begin{table}[H]
\centering
\small
\setlength{\tabcolsep}{3pt}
\renewcommand{\arraystretch}{1.15}
\rowcolors{2}{black!3}{white}
\begin{adjustbox}{max width=\linewidth,center}
\begin{tabular}{@{} L{0.40\linewidth} C{0.12\linewidth} C{0.18\linewidth} C{0.22\linewidth} @{}}
\toprule
\textbf{Environment} & \textbf{Agents} & \textbf{Episode horizon} & \textbf{Initial budget} \\
\midrule
\env{Pursuit-Evasion} Classic 2 & $2$ & $100$ & $1{,}000{,}192$ \\
\env{Pursuit-Evasion} Capture 2 & $2$ & $100$ & $1{,}000{,}192$ \\
\env{Pursuit-Evasion} Classic 3 & $3$ & $100$ & $1{,}500{,}160$ \\
\env{Pursuit-Evasion} Capture 3 & $3$ & $100$ & $1{,}500{,}160$ \\
\env{Coins} Own Coin First & $2$ & $256$ & $200{,}192$ \\
\env{Matrix Stag Hunt} & $2$ & $20$ & $500{,}224$ \\
\env{Commons Harvest} Safe Harvest & $4$ & $1000$ & $200{,}192$ \\
\env{Territory} Role Claim & $2$ & $256$ & $200{,}192$ \\
\bottomrule
\end{tabular}
\end{adjustbox}
\rowcolors{2}{}{}
\caption{Episode horizons and initial training budgets in environment steps, before the optional $25\%$ extension.}
\label{tab:hyperparams_protocol}
\end{table}

\subsection{Learner Settings}

\begin{table}[H]
\centering
\small
\setlength{\tabcolsep}{5pt}
\renewcommand{\arraystretch}{1.15}
\rowcolors{2}{black!3}{white}
\begin{adjustbox}{max width=\linewidth,center}
\begin{tabular}{@{} L{0.24\linewidth} L{0.68\linewidth} @{}}
\toprule
\textbf{Setting} & \textbf{Value} \\
\midrule
Applicable algorithms & \algo{IPPO}, \algo{PR2-IPPO}, and their CER and handcrafted variants. \\
Rollout batch & $256$ transitions; environment counts and chunk lengths are in Table~\ref{tab:hyperparams_rollout_replay}. \\
Policy/value networks & Two-layer MLPs, hidden widths $[64,64]$, \texttt{tanh} activations. \algo{PR2-IPPO} uses a joint-action critic. \\
Optimizer & Adam, learning rate \(3\times10^{-4}\) (\(10^{-4}\) for Coins); annealing only for Safe Harvest; gradient-norm clipping $0.5$. \\
PPO objective & GAE parameter \(\lambda=1\) for Safe Harvest and \(0.95\) otherwise; clipping parameter \(0.2\), value-loss coefficient \(0.5\), entropy coefficient \(0.01\). \\
Update schedule & $4$ epochs and $4$ minibatches per rollout; $64$ transitions per minibatch before CER expansion. \\
Coins CER replay & \texttt{uncorrected} mode. \\
\bottomrule
\end{tabular}
\end{adjustbox}
\rowcolors{2}{}{}
\caption{Policy-gradient hyperparameters for IPPO-family algorithms.}
\label{tab:hyperparams_ippo}
\end{table}

\begin{table*}[t]
\centering
\small
\setlength{\tabcolsep}{4pt}
\renewcommand{\arraystretch}{1.15}
\rowcolors{2}{black!3}{white}
\begin{adjustbox}{max width=\linewidth,center}
\begin{tabular}{@{} L{0.22\linewidth} L{0.40\linewidth} L{0.30\linewidth} @{}}
\toprule
\textbf{Setting} & \textbf{\algo{IQL} / \algo{PR2-IQL}} & \textbf{\algo{Nash-Q}} \\
\midrule
Applicable variants & \algo{IQL}, \algo{PR2-IQL}, and their CER and handcrafted variants. & \algo{Nash-Q}, \algo{Nash-Q-CER}, \algo{hc-Nash-Q} (Matrix Stag Hunt). \\
Function class & Recurrent Q-network, hidden size $64$, \texttt{tanh} activation; joint-action critic for \algo{PR2-IQL}. & Tabular joint-action Q-values indexed by joint environment and monitor observations and elapsed time. \\
Optimizer / step size & Adam, learning rate \(3\times10^{-4}\), gradient-norm clipping $0.5$; $4$ epochs per replay batch. & Visit-count step size \(\alpha(x,\bm a)=1/N(x,\bm a)\). \\
Target network & Update interval $25$ trainer updates; target mixing \(\tau=1.0\). & -- \\
Replay and exploration & Buffer capacity $25{,}000$ chunks; learning starts after $1024$ steps; linear \(\varepsilon\)-greedy decay from $1$ to $0.05$ over the first half of training. & Linear \(\varepsilon\)-greedy decay from $1$ to $0.05$ over the first half of training. \\
\bottomrule
\end{tabular}
\end{adjustbox}
\rowcolors{2}{}{}
\caption{Value-learning hyperparameters. Nash-Q settings apply to Matrix Stag Hunt.}
\label{tab:hyperparams_iql}
\end{table*}

\paragraph*{Nash-Q convergence.}
The $1/N$ schedule in Table~\ref{tab:hyperparams_iql}, counting CER updates, satisfies Theorem~\ref{thm:nashq_convergence}(ii) under infinitely many updates. The theorem also requires its sampling and selector assumptions and concerns temporal utilities; the handcrafted baseline uses native stage rewards.

\begin{table}[H]
\centering
\small
\setlength{\tabcolsep}{3pt}
\renewcommand{\arraystretch}{1.15}
\rowcolors{2}{black!3}{white}
\begin{adjustbox}{max width=\linewidth,center}
\begin{tabular}{@{} L{0.32\linewidth} C{0.16\linewidth} C{0.16\linewidth} C{0.28\linewidth} @{}}
\toprule
\textbf{Benchmark} & \textbf{Parallel envs.} & \textbf{Chunk length} & \textbf{IQL replay batch} \\
\midrule
\env{Stochastic Pursuit-Evasion} & $4$ & $64$ & $16$ chunks \\
\env{Matrix Stag Hunt} & $8$ & $32$ & $32$ chunks \\
\env{Coins} & $1$ & $256$ & $4$ chunks \\
\env{Commons Harvest} & $1$ & $256$ & $4$ chunks \\
\env{Territory} & $1$ & $256$ & $4$ chunks \\
\bottomrule
\end{tabular}
\end{adjustbox}
\rowcolors{2}{}{}
\caption{Rollout settings for all learners and IQL replay batches ($1024$ transitions each).}
\label{tab:hyperparams_rollout_replay}
\end{table}

\subsection{Evaluation Settings}

\begin{table}[H]
\centering
\small
\setlength{\tabcolsep}{5pt}
\renewcommand{\arraystretch}{1.15}
\rowcolors{2}{black!3}{white}
\begin{adjustbox}{max width=\linewidth,center}
\begin{tabular}{@{} L{0.34\linewidth} C{0.20\linewidth} C{0.20\linewidth} C{0.20\linewidth} @{}}
\toprule
\textbf{Benchmark} & \textbf{NashConv episodes} & \textbf{BR train iters} & \textbf{BR batch steps} \\
\midrule
\env{Pursuit-Evasion} (2-player) & $64$ & $6$ & $1024$ \\
\env{Pursuit-Evasion} (3-player) & $96$ & $8$ & $1024$ \\
\env{Matrix Stag Hunt} & $50$ & $4$ & $512$ \\
\env{Coins} & $32$ & $3$ & $512$ \\
\env{Commons Harvest} & $32$ & $3$ & $512$ \\
\env{Territory} & $32$ & $3$ & $512$ \\
\bottomrule
\end{tabular}
\end{adjustbox}
\rowcolors{2}{}{}
\caption{Unilateral response-search budgets for \ApxNashConv. Scores above $0.05$ trigger up to $25\%$ additional profile training.}
\label{tab:hyperparams_eval}
\end{table}

\end{document}